\documentclass[letterpaper, 10 pt, conference]{ieeeconf}  

\usepackage{etoolbox}
\newtoggle{ext}     
\toggletrue{ext}    

\usepackage{amsmath}
\usepackage{amssymb}
\usepackage{amsthm}
\usepackage{bm}
\usepackage{dsfont}

\usepackage[ruled,vlined,linesnumbered]{algorithm2e}
\usepackage{algpseudocode}

\usepackage{mathtools,cuted}
\usepackage{mathrsfs}

\usepackage{graphicx}
\usepackage{subcaption}
\usepackage[font=footnotesize]{caption}
\usepackage{hhline}
\usepackage{array}

\usepackage{enumerate}

\usepackage{stfloats}
\usepackage{siunitx}
\usepackage[usenames, dvipsnames]{color}
\usepackage[noadjust]{cite}
\usepackage[linkcolor=blue,citecolor=blue,urlcolor=blue]{hyperref}
\usepackage[capitalise]{cleveref}
\usepackage[colorinlistoftodos]{todonotes}

\usepackage{gensymb}

\usepackage{paralist}
\usepackage{todonotes}

\newtheorem{definition}{Definition}[section]
\newtheorem{theorem}{Theorem}[section]
\newtheorem{proposition}{Proposition}[section]

\newtheorem{lemma}{Lemma}[section]

\usepackage{tikz-cd}
\tikzcdset{every label/.append style = {font = \small}}

\usepackage{graphicx}

\usepackage{stackengine}

\IEEEoverridecommandlockouts                              
\useRomanappendicesfalse                                

\renewcommand{\baselinestretch}{0.8928}     
\title{\LARGE \bf Composable Geometric Motion Policies\\
using Multi-Task Pullback Bundle Dynamical Systems
}

\author{Andrew Bylard, Riccardo Bonalli, Marco Pavone 
\thanks{A. Bylard, R. Bonalli, and M. Pavone are with the Department of Aeronautics and Astronautics, Stanford University, Stanford, CA 94305. \{{\tt bylard, rbonalli, pavone}\} {\tt@stanford.edu}. This work was supported in part by NSF Cyber-Physical Systems (CPS) Grant No. 1931815 and by KACST. Thanks to Fransesco Bullo, Benoit Landry, Thomas Lew, and Jean-Jacques Slotine for their helpful discussions and insights.}
}

\begin{document}

\maketitle
  \thispagestyle{empty}
\pagestyle{empty}

\begin{abstract}
Despite decades of work in fast reactive planning and control, challenges remain in developing reactive motion policies on non-Euclidean manifolds and enforcing constraints while avoiding undesirable potential function local minima. This work presents a principled method for designing and fusing desired robot task behaviors into a stable robot motion policy, leveraging the geometric structure of non-Euclidean manifolds, which are prevalent in robot configuration and task spaces. Our Pullback Bundle Dynamical Systems (PBDS) framework drives desired task behaviors and prioritizes tasks using separate position-dependent and position/velocity-dependent Riemannian metrics, respectively, thus simplifying individual task design and modular composition of tasks. For enforcing constraints, we provide a class of metric-based tasks, eliminating local minima by imposing non-conflicting potential functions only for goal region attraction. We also provide a geometric optimization problem for combining tasks inspired by Riemannian Motion Policies (RMPs) that reduces to a simple least-squares problem, and we show that our approach is geometrically well-defined. We demonstrate the PBDS framework on the sphere $\mathbb S^2$ and at 300-500 Hz on a manipulator arm, and we provide task design guidance and an open-source Julia library implementation. Overall, this work presents a fast, easy-to-use framework for generating motion policies without unwanted potential function local minima on general manifolds.
\end{abstract}

\section{Introduction}
\label{sec:introduction}
Fast reactive planning and control is often crucial in dynamic, uncertain environments, and related techniques have a long history~\cite{Khatib1985,KhoslaVolpe1988,Hogan1984,Khansari-ZadehBillard2012,IjspeertNakanishiEtAl2013}. However, key challenges remain in applying these techniques to general manifolds. Well-known non-Euclidean manifolds are ubiquitous in robotics, often representing part of the configuration manifold (e.g., $SO(3)$ for aerial robot attitude). They can also be essential for representing robotic task spaces (e.g., the sphere $\mathbb{S}^2$ for painting or moving a fingertip around an object surface, such as in Fig. \ref{fig:main}).
Further, constraints may restrict the robot to a free submanifold
which is difficult to represent explicitly. For example, an obstacle in the workspace can produce
an additional topological hole that geodesics, or ``default" unforced trajectories, should avoid~\cite{RatliffToussaintEtAl2015,AgirrebeitiaAvilesEtAl2005}.

Often such constraints are handled using constrained optimization or repulsive potentials. However, constrained optimization can scale poorly with the number and complexity of nonconvex constraints, often becoming impractical for real-time MPC without careful initialization~\cite{SchulmanHoEtAl2013,KuindersmaDeitsEtAl2015,MerktIvanEtAl2019}. On the other hand, repulsive potentials can be difficult to design without producing incorrect behavior. For example, it is well known that artificial potential function (APF) techniques~\cite{Khatib1985,KhoslaVolpe1988} without assumptions about the robot and workspace geometry~\cite{RimonKoditschek1992,KimKhosla1992,HuberBillardEtAl2019} can produce many incorrect local minima and unnatural behavior such as oscillations within the free submanifold (e.g. in narrowing passages or when a constraint is near the goal)~\cite{KorenBorenstein1991,Kim2009}.

An alternative coming from differential geometry is to encode constraints not with forces, but with {\em metrics}. For example, correctly designed Riemannian metrics~\cite{Lee2018} defined on the robot configuration manifold have been proposed to curve the manifold to prevent constraint violation, not due to forces pushing the robot away but due to the space stretching infinitely in the direction of constraints ~\cite{ChengMukadamEtAl2018,MainpriceRatliffEtAl2020}. Such a reliance on curvature rather than competing potential functions may also eliminate traps due to potential function local minima~\cite{RatliffVanWykEtAl2020}.

\begin{figure}[t!]
    \centering
    \includegraphics[width=\linewidth]{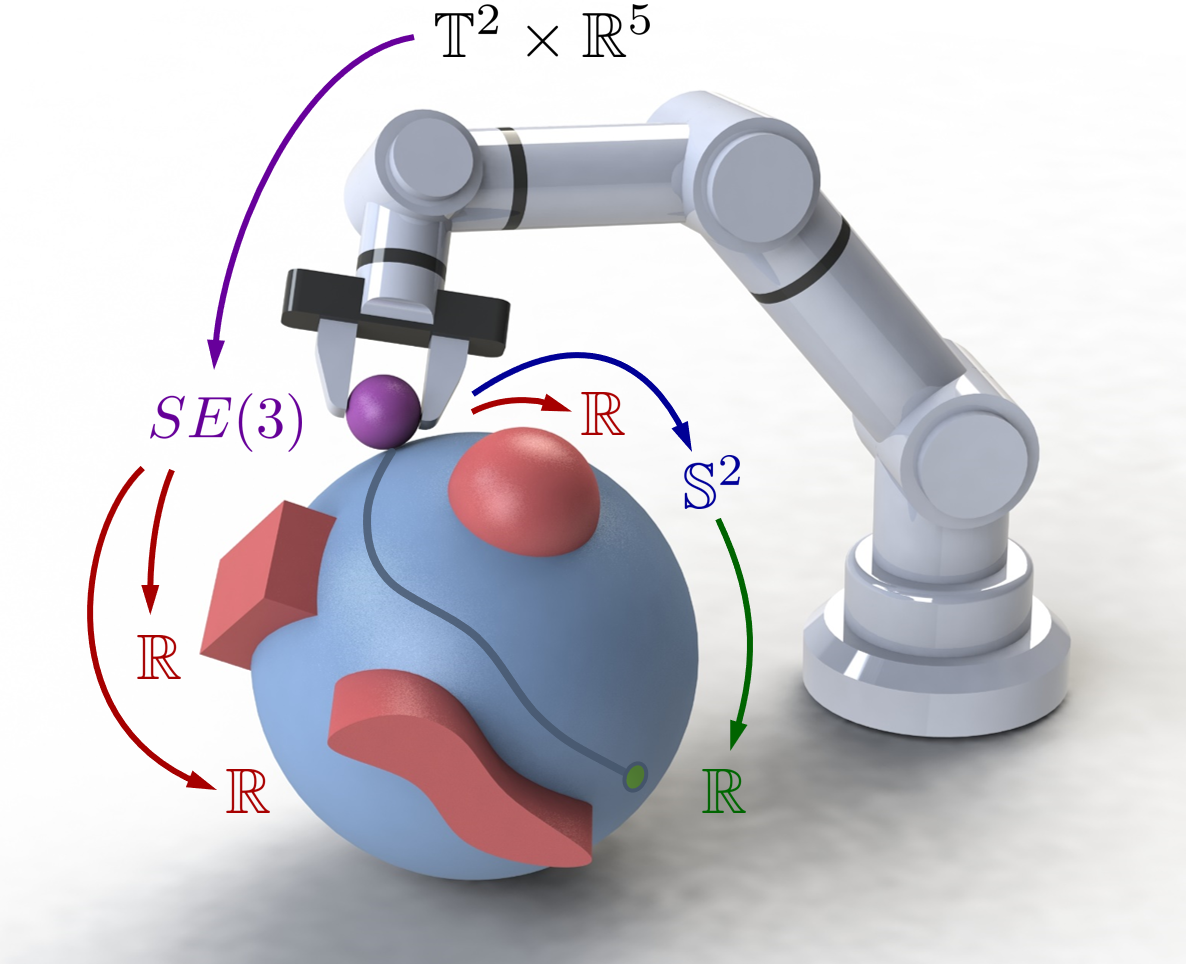}%
    \caption{Example tree of PBDS task mappings designed to move a ball along the surface of a sphere to a goal while avoiding obstacles. Depicted are manifolds representing: (black) joint configuration for a 7-DoF robot arm with two fully revolute joints, (purple) ball pose, (red) distances to obstacles, (blue) position on the sphere, and (green) distance to the goal. Note that damping can be defined on  $\mathbb T^2 \times \mathbb R^5$, $SE(3)$, and/or $\mathbb S^2$ as desired.}
    \label{fig:main}%
\end{figure}

However, correctly designing such metrics directly on the configuration manifold can be difficult, motivating the use of \textit{pullback metrics}, i.e. metrics are designed on task manifolds where constraints are naturally defined and then ``pulled back" through the task map onto the configuration manifold. One recent work advocating this approach is RMPflow~\cite{ChengMukadamEtAl2018}, which develops a tree of subtask manifolds with metrics and forces defined on the leaf manifolds which are pulled back and combined into a single acceleration policy on the configuration manifold. However, \cite{ChengMukadamEtAl2018} uses potentials in addition to metrics to represent constraint-enforcing subtasks, increasing design complexity and leading to the same local minima issues as in other APF approaches.

Additionally, it is often useful to perform task prioritization dependent on robot velocity. For example, when a robot is moving away from an obstacle, the obstacle can often be safely ignored. To this end, \cite{ChengMukadamEtAl2018} introduces Geometric Dynamical Systems (GDSs) weighted by Riemannian Motion Policies (RMPs)~\cite{RatliffIssacEtAl2018}, which allows velocity-dependent metrics used both to specify subtask behavior (such as constraint avoidance) and to prioritize tasks. However, this approach presents two difficulties: First, it is very difficult to design task metrics which both specify desirable task behavior and prioritize tasks correctly. This has led to investigating learning the task metrics from demonstrations~\cite{RanaLiEtAl2020}, but in any case, it is unclear what task behavior within a single task could benefit from velocity-dependent Riemannian metrics, as we discuss in Sec.~\mbox{\ref{subsec:task_priorities_vel_dep_metrics}}.

Second, a GDS with velocity-dependent metrics does not meet the key property of \emph{geometric consistency}, as we demonstrate in Sec.~\ref{subsec:geo_consistency}. Geometric consistency\footnote{Geometrically-consistent objects are also known of as being ``global" in differential geometry~\cite{Lee2013}. Thus we will also refer to such objects as being geometrically or globally well-defined.} (i.e. invariance to changes in coordinate representations) is an essential property of any differential geometric method which ensures that the quantities defined, in this case acceleration policies, are well-defined on the robot and task manifolds invoked, thus correctly capturing and leveraging the structure of these manifolds (for example, see Fig. \ref{fig:geo_consistency}).

To summarize, a large gap remains for producing a fast, easy-to-use, and geometrically-consistent approach for generating motion policies on general robot and task manifolds, achieving velocity-dependent task-weighting, and leveraging metric-based enforcement of constraints to eliminate unwanted potential function local minima.

{\em Statement of Contributions}:
To this end, we provide the following contributions:

1) We present the Pullback Bundle Dynamical Systems (PBDS) framework for combining multiple geometric task behaviors into a single robot motion policy while maintaining geometric consistency and stability. In doing so, we provide a geometrically well-defined formulation of a weighting scheme inspired by RMPs, which reduces to a simple least-squares problem. We also remove the tension between single-task behavior and inter-task weighting by introducing separate velocity-dependent weighting pseudometrics for each task to handle task prioritization.

2) We apply the PBDS framework to tangent bundles to reveal limits to the practical use of velocity-dependent Riemannian metrics for task behavior design and to show why GDSs do not maintain geometric consistency.

3) We provide a class of constraint-enforcing tasks encoded solely via simple, analytical Riemannian metrics that stretch the space, rather than via traditional barrier function potentials, eliminating potential function local minima. 

4) We demonstrate PBDS policy behavior in numerical experiments and at 300-500 Hz on a 7-DoF arm, and we provide a fast open-source Julia library called PBDS.jl.

{\em Paper Organization}:
The paper is organized as follows. In Sec.~\ref{sec:geo_prel} we recall relevant concepts from Riemannian geometry and establish some notation. Then Sec.~\ref{sec:pbds} builds a multi-task PBDS, extracts a motion policy, and provides stability results. Sec.~\ref{sec:gds_comparison} unpacks key advantages over RMPflow's GDS/RMP framework. Finally, Sec.~\ref{sec:experiments} provides implementation details and a robot arm demonstration.


\section{Geometric Preliminaries}
\label{sec:geo_prel}
Here we recall some results in Riemannian geometry and establish notation that will be used throughout the paper. For a more detailed introduction, see~\cite{Lee2013,Lee2018}. Let $M$ be a smooth $m$-dimensional manifold equipped with charts assigning coordinates to locally-Euclidean patches. $M$ has a smooth tangent bundle $TM$ containing tangent spaces $T_p M$ at each point $p \in M$, and it has a cotangent bundle $T^* M$, which is a natural place to define forces~\cite{BulloLewis2004}. We will also denote $(p,v) \in T_pM$ as $v_p$ or $v$ depending on desired emphasis of the base point $p$. Given a Riemannian metric $g$, the couple $(M,g)$ is a Riemannian manifold. The metric provides a smoothly-varying inner product $g_p : T_p M \times T_p M \longrightarrow \mathbb{R}$ at each point and thus induces norms $||\cdot||_{g_p}$. The metric $g$ also gives a ``sharp" operator $\sharp : T^* M \longrightarrow TM$ and generalized gradient $\textrm{grad} f : M \longrightarrow TM$ for $f \in C^\infty(M)$, which are useful for applying forces. A choice of connection $\nabla$ in $TM$ (e.g. the standard Levi-Civita connection) assigns a unique acceleration operator $D_\sigma$ to each curve $\sigma$ in $M$. Thus we can compute acceleration $D_\sigma \sigma' = \nabla_{\sigma'(t)}\sigma'$, where $\sigma'$ is the velocity along $\sigma$ (used interchangeably with $\dot\sigma$ notation, which can also indicate a time derivative). Then by choosing generalized forces $\mathcal{F} : TM \longrightarrow T^*M$, one can specify a dynamical system on $M$ satisfying $D_\sigma \sigma'(t) = \mathcal{F}(\sigma(t))^\sharp$. This can be written in local coordinates as the second-order differential equations
\begin{equation}
\label{eq:dsmanifold}
\ddot \sigma^k(t) + \dot \sigma^i(t) \dot \sigma^j(t) \Gamma^k_{ij}(\sigma(t)) = g^{kj}(\sigma(t)) \mathcal{F}_j(\sigma(t),\dot \sigma(t)),\hspace{-1pt}
\end{equation}
where we make use of Einstein index notation and where $\Gamma^k_{ij}$ are the Christoffel symbols associated with $\nabla$. We will use this type of dynamical system to design desired behavior on task manifolds that corresponding robot motion policies should aim to replicate.

Note also that we will use bold symbols when convenient to denote local matrix or vector representations of objects. For example, given $\Xi^k \triangleq \dot\sigma^i \dot\sigma^j \Gamma^k_{ij}$, we can write (\ref{eq:dsmanifold}) as
$$
\ddot{\bm{\sigma}}(t) + \bm{\Xi}(\bm \sigma(t), \dot{\bm{\sigma}}(t)) = \bm{g}^{-1}(\bm\sigma(t))\bm{\mathcal{F}}(\bm\sigma(t), \dot{\bm{\sigma}}(t)).
$$
We may also abbreviate such expressions, for example as
$$
\ddot{\bm{\sigma}} + \bm{\Xi} = \bm{g}^{-1}\bm{\mathcal{F}}.
$$\vspace{-20pt}

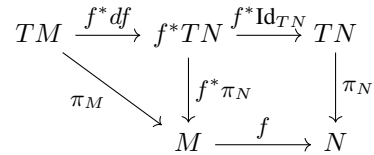
\begin{figure}
    \centering
\begin{tikzcd}[row sep=25pt, column sep=25pt]
TM \arrow[r, "f^*df"] \arrow[rd, "\pi_M"'] & f^*TN \arrow[d, "f^*\pi_N"] \arrow[r, "f^*\textrm{Id}_{TN}"] & TN \arrow[d, "\pi_N"] \\
                                                & M \arrow[r, "f"]                               & N                    
\end{tikzcd}
    \caption{Commutative diagram of the key manifolds used to form a pullback bundle dynamical system. Note that the pullback bundle $f^*TN$ is an $n$-vector bundle over the robot configuration manifold $M$.}
    \label{fig:commdiag_pbds}
\end{figure}


\section{Pullback Bundle Dynamical Systems}
\label{sec:pbds}
Given $K$ tasks, let $M$ and $N_i$ be smooth $m$- and $n_i$-dimensional robot configuration and task manifolds, respectively, where $\{f_i : M \longrightarrow N_i\}_{i=1,\ldots,K}$ is a set of smooth task maps.
These tasks could represent a variety of desired behaviors such as attraction to a point/region, velocity damping, or avoidance of a constraint. To illustrate, consider a running pedagogical example of a point robot on the surface of a ball, aiming to reach a goal while avoiding obstacles. A simple design for the relevant tasks is shown in Table \ref{tab:example}. As shown, the desired task behaviors can be encoded by designing on each task manifold $N_i$ a potential function $\Phi_i: N_i \longrightarrow \mathbb{R}_+$, dissipative forces $\mathcal F_{D,i} : TN_i \longrightarrow T^*N_i$, and a Riemannian metric $g_i$ (which we can refer to as the behavior metric), all of which are smooth. We also choose the standard Levi-Civita connection $\nabla_i$~\cite{Lee2013}. \begin{table*}[t!]
    \centering
    \begin{center}
    \begin{tabular}{|c|c|c|c|c|c|}
    \hline
     Task & Task Map & Behavior Metric & Dissipative Forces & Potential & Weight Pseudometric\\ \hline
     Goal Attraction & $f_1 : \mathbb S^2 \rightarrow \mathbb R$ : $p \mapsto \textrm{dist}(p,p_\textrm{goal})$ & $\bm g_1 = 1$ & $\bm{\mathcal F}_{D,1} = 0$ & $\Phi_1 = \lVert \bm x \rVert^2_2$ & $\bm w_i = \bm I_2$\\  
     Damping & $f_2 : \mathbb S^2 \rightarrow \mathbb S^2$ : $p \mapsto p$ & $\bar{\bm g}_2 = \bm I_3$ & $\bar{\bm{\mathcal F}}_{D,2} = -4\dot{\bar{\bm x}}$ & $\Phi_2 = 0$ & $\bar{\bm w}_2 = \bm I_6$ \\
     Obstacle Avoidance & $f_i : \mathbb S^2 \rightarrow \mathbb R_+$ : $p \mapsto \textrm{dist}(p,\mathcal X_\textrm{obs})$ & $\bm g_i = \exp(1/(2\bm x^2))$ & $\bm{\mathcal F}_{D,i} = 0$ & $\Phi_i = 0$ & See (\ref{eq:toggled_pseudometric})\\
     \hline
    \end{tabular}
    \end{center}
    \caption{Simple task design summary for running example of a robot on a sphere surface navigating to a goal point while avoiding obstacles. Note that each obstacle has an avoidance tasks (hence the indices $i$). For the damping task, the overbars denote representations in ambient Euclidean space, which are easily pulled back through standard embedding maps to induce corresponding objects on $\mathbb S^2$ and $T\mathbb S^2$ (see \iftoggle{ext}{Appx.~\ref{subsec:S2_tasks}}{Appendix \cite{BylardBonalliEtAl2021}} for details).}
    \label{tab:example}
    \vspace{-16pt}
\end{table*}

As mentioned, these components give dynamical systems on the task manifolds $N_i$, but we want to form a robot motion policy from a dynamical system on the configuration manifold $M$. One way to proceed is by using pullbacks, i.e. ``pulling back" the necessary objects through the task maps $f_i$ to operate over $M$. First, we recall the definition of the pullback bundle of $TN$, omitting task indices for simplicity:
\begin{equation}
\label{eq:pullback_bundle}
f^*TN = \coprod_{p \in M} \pi_N^{-1}(f(p)),
\end{equation}
where $\coprod$ is a disjoint union and $\pi_N$ is the standard projection on $TN$. This is itself a well-defined vector bundle and in particular is an $(m+n)$-dimensional manifold. The relationships between these manifolds are depicted in Fig~\ref{fig:commdiag_pbds}, including a pullback differential $f^*df : TM \longrightarrow f^*TN : (p,v) \mapsto (p, \pi_2(df_p(v)))$, where $\pi_2$ denotes a projection onto the velocity component.

Next we define the pullback connection $f^*\nabla$, whose Christoffel symbols are given locally by
\begin{equation}
f^* \Gamma^k_{ij}(p) = \frac{\partial f^\ell}{\partial x^i}(p) \Gamma^k_{\ell j}(f(p))
\end{equation}
for all $p \in M$. In \iftoggle{ext}{Lemma \ref{lem:pullback_connection}}{the Appendix~\cite{BylardBonalliEtAl2021}} we show this connection is globally well-defined and compatible with a pullback metric $f^*g$.
This connection gives a corresponding pullback acceleration operator $f^*D_\gamma$ for each curve $\gamma$ in $M$. We can also denote the total acceleration given by the pullback forces as $f^*\mathcal F(\cdot)^\sharp : f^*TN \longrightarrow f^*TN : (p,v) \mapsto f^*\mathcal F_D(p,v)^\sharp - f^*\textrm{grad}\,\Phi(p)$, where we define the pullback dissipation forces $f^*\mathcal F_D$ and pullback gradient $f^*\textrm{grad}$ in \iftoggle{ext}{Appx.~\ref{subsec:constructions_pb}}{the Appendix~\cite{BylardBonalliEtAl2021}}.

We are now equipped to construct a key building block for forming a full Pullback Bundle Dynamical System:
\begin{definition}[Local Pullback Bundle Dynamical System] Let $f : M \longrightarrow N$ be a smooth task map. Then for each $(p,v)\in M$, we can choose a curve $\alpha_{(p,v)} : (-\varepsilon, \varepsilon) \longrightarrow M$ resulting in $\gamma_{\alpha_{(p,v)}} : (-\varepsilon, \varepsilon) \longrightarrow f^*TN$ for some $\varepsilon > 0$ such that $(f,g,\Phi,\mathcal F_D, \alpha_{(p,v)})$ forms a local Pullback Bundle Dynamical System (PBDS) satisfying 
\begin{equation*}
\textrm{PBDS}_{\alpha_{(p,v)}} \hspace{-5pt}\ \begin{cases}
    f^* D_{\alpha_{(p,v)}} \gamma_{\alpha_{(p,v)}}(s) = f^*\mathcal{F}(\gamma_{\alpha_{(p,v)}}(s))^\sharp\\
    \gamma_{\alpha_{(p,v)}}(0) = f^*df_{\alpha_{(p,v)}(0)}(\alpha_{(p,v)}'(0)),\\ \alpha_{(p,v)}'(0) = (p,v).
\end{cases}
\end{equation*}
\end{definition}
This local PBDS construction is local in the sense that in practice, we define a $\textrm{PBDS}_{\alpha_{(p,v)}}$ for each $(p,v) \in TM$ and only evaluate it at $(p,v)$. However, it is geometrically well-defined and is required to ensure that we have well-defined dynamics at each point on the pullback bundle independent of the robot curve $\sigma : [0,\infty) \longrightarrow M$ that we ultimately follow.
In particular, these dynamics must be well-posed at $t=0$, (i.e. where $f^*D_\sigma$ is not well-defined, requiring another curve $\alpha_{\sigma'(0)}$ to produce $f^*D_{\alpha_{\sigma'(0)}}$)
and in cases where the curve $\alpha_{\sigma'(t)}$ defining a valid $\textrm{PBDS}_{\alpha_{\sigma'(t)}}$ for some $t \in [0, \infty)$ is not unique. For example, the latter may occur when there is redundancy, i.e., $m > n$, which is often the case in practice.

From these local PBDS dynamics, we can extract the corresponding desired task pullback acceleration associated with each robot position and velocity:
\begin{equation}
\label{eq:S}
S : TM \longrightarrow T(f^*TN) : (p,v) \mapsto \pi_\textrm{VB}(\gamma'_{\alpha_{(p,v)}}(0)),
\end{equation}
where $\pi_\textrm{VB}$ denotes the projection onto the vertical bundle. In \iftoggle{ext}{Appx.~\ref{appxsubsec:multi_pbds}}{the Appendix~\cite{BylardBonalliEtAl2021}}, we show that this map is geometrically well-defined. In particular, a curve $\alpha_{(p,v)}$ forming a valid local PBDS always exists, and the map $S$ does not depend on the particular choices of curves $\alpha_{(p,v)}$.
Given this fact, we will omit further mention of $\alpha$ and use the shorthand $\gamma'_{v_p}(0) \triangleq \gamma'_{\alpha_{(p,v)}}(0)$ for $(p,v) \in TM$.

Now for each individual task, we can form a map $S_i$ as above to retrieve corresponding desired task pullback accelerations. However, our goal is to combine these tasks into a single robot motion policy. To do this, one approach is to define a robot acceleration policy using a geometrically well-defined optimization problem designed to strike a weighted balance between these pullback task acceleration policies.

This is not as straightforward as it may appear since the operative accelerations are each in different spaces and are thus difficult to compare (i.e. $\ddot\sigma$ is in $TTM$ and $\gamma'_{\dot\sigma(t)}$ is in $T(f_i^*TN_i)$ for each task).
To resolve this, we form a map relating robot accelerations on $TTM$ to their resulting pullback task accelerations in $T(f_i^*TN_i)$:
\begin{equation}
\begin{aligned}
\label{eq:Zi}
Z_i : TTM &\longrightarrow T(f_i^*T N_i) \\
\big( (p,v) , a \big) &\mapsto \pi_{\textrm{VB}}\Big( d(f_i^*df_i)_{(p,v)}(a) \Big).
\end{aligned}
\end{equation}

Next, for each task we define a Riemannian pseudometric $w_i$ on $TN_i$ which will provide its weighting against other tasks, and we impose the following assumptions:
\begin{itemize}
    \item[$(A1)$] The pseudometrics $w_i$ are at every point in $TN_i$ either positive-definite or zero.
    \item[$(A2)$] The Jacobian of the product map of the task maps associated with nonzero weights has rank $m$.
\end{itemize}
In practice, these weighting pseudometrics are often trivial to design (e.g., see our running example in Table \ref{tab:example}, where the attractor task weight is the Euclidean metric, and the damping task weight is induced by a Euclidean metric). However, they can also be used to switch tasks on and off, which is useful for enforcing constraints as shown in Sec.~\ref{subsec:constraints}. Also, as in the running example, $(A1)$ and $(A2)$ are easy to satisfy. In particular, it is typical to use one or more damping tasks which together always provide dissipation along all degrees of freedom of the robot.
Now to conclude, each $w_i$ also has a pullback pseudometric $F_i^*w_i$ on $T(f_i^*T N_i)$ defined using the natural higher-order task map $F_i : TM \rightarrow TN_i : (p,v) \mapsto (p, (df_i)_p(v))$ (\iftoggle{ext}{Appx.~\ref{appxsubsec:multi_pbds}}{see Appx.~\cite{BylardBonalliEtAl2021}}).

We can now form the desired ODE on our robot manifold:
\begin{definition}[Multi-Task Pullback Bundle Dynamical System] Let $\{f_i : M \longrightarrow N_i\}_{i=1,...K}$ be smooth task maps. Then the set $\{(f_i, g_i, \Phi_i, \mathcal F_{D,i}, w_i)\}_{i=1,\ldots,K}$ forms a multi-task PBDS with curves $\sigma : [0,\infty) \longrightarrow M$ satisfying
\begin{equation}
\label{eq:multi_pbds_def}
\begin{cases}
\ddot{\sigma}(t) = \underset{a \in \mathcal D_{\dot\sigma(t)}}{\arg \min} \ \sum^K_{i=1} \frac{1}{2} \lVert Z_i(a) - S_i(\dot\sigma(t))   \rVert^2_{F_i^*w_i} \\
\dot\sigma(0) = (p_0, v_0).
\end{cases}
\end{equation}
\end{definition}
\noindent Here, $\mathcal D$ is the globally well-defined affine distribution of $TTM$ such that the subspace $\mathcal D_{(p,v)} \subseteq T_{(p,v)}$ satisfies $a^v = v$ componentwise for each $a = ((p, v), (a^v, a^a)) \in \mathcal D_{(p,v)}$. Under assumptions $(A1)$-$(A2)$, as shown in \iftoggle{ext}{Appx.~\ref{appxsubsec:multi_pbds}}{the Appendix~\cite{BylardBonalliEtAl2021}}, the multi-task PBDS is geometrically well-defined and has a unique smooth solution. This gives us a dynamical system on $M$ that provides our robot acceleration policy:
\begin{equation}
\label{eq:local_pbds_policy}
\ddot{\bm\sigma} = \left(\sum_{i=1}^K\bm{Jf}_i^\top\bm{w}_i^a\bm{Jf}_i\right)^\dagger\left(\sum_{i=1}^K \bm{Jf}_i^\top\bm{w}_i^a\bm{\mathcal{A}}_i\right)
\end{equation}\vspace{-5pt}
\begin{equation}
\begin{aligned}
\label{eq:A}
\bm{\mathcal A}_i = \bm g_i^{-1} (&\bm{\mathcal F}_{D,i} - \nabla \bm\Phi_i) - (\dot{\bm{Jf}}_i-\bm\Xi_i)\dot{\bm\sigma} \\
(\Xi_i)_{kj} &= (Jf_i)_{\ell j}(\Gamma_i)^k_{\ell h}(Jf_i)_{hr}\dot\sigma^r,
\end{aligned}
\end{equation}
where $\nabla$ in this context is the Euclidean gradient operator, and where $\bm{w}_i^a \in \mathbb{R}^{n_i \times n_i}$ is the lower-right quadrant of the local matrix $\bm{w}_i\in \mathbb{R}^{2n_i \times 2n_i}$. In practice, this is the only quadrant necessary to design due to the vertical bundle projections in (\ref{eq:S}) and (\ref{eq:Zi}).

Next, we use LaSalle's invariance principle to demonstrate stability about a set of robot equilibrium states. This requires a Lyapunov function, which we build adapting known results in Lagrangian mechanics to the multi-task PBDS. In particular, we rely on one more key assumption:
\begin{itemize}
    \item[$(A3)$] The combined dissipative forces corresponding to tasks having nonzero weights are strictly dissipative.
\end{itemize}
Like the others, $(A3)$ can be naturally satisfied (see the discussion of $(A2)$). Now consider the Lyapunov function candidate on the robot tangent bundle $TM$:
\begin{equation}
V(p,v) = \sum^K_{i=1}\frac{1}{2} \big\lVert (f_i^*df_i)_p(v) \big\rVert^2_{(f_i^* g_i)_p} + \Phi_i \circ f_i(p).
\end{equation}
In \iftoggle{ext}{Appx.~\ref{appxsubsec:multi_pbds}}{the Appendix~\cite{BylardBonalliEtAl2021}}, we show that $\dot V(\sigma'(t)) < 0$ outside equilibrium states, leading to the following stability results:

\begin{theorem}[Global Stability of Multi-Task PBDS]
Let  $\{(f_i, g_i, \Phi_i, \mathcal F_{D,i}, w_i)\}_{i=1,\ldots,K}$ be a multi-task PBDS where $\sum_{i=1}^K \Phi_i \circ f_i$ is a proper map.
Then given the Lyapunov function $V$ above, the multi-task PBDS satisfies these:
\begin{itemize}
    \item The sublevel sets of $V$ are compact and positively invariant for \eqref{eq:multi_pbds_def}, so that every solution curve $\sigma$ is defined in the whole interval $[0,+\infty)$.
    \item For every $\beta > 0$, every solution curve $\sigma$ in $M$ starting at $\dot{\sigma}(0) \in V^{-1}([0,\beta])$ converges in $V^{-1}([0,\beta])$ as $t\rightarrow +\infty$ to an equilibrium set $\{(p,0)\in V^{-1}([0,\beta])$ : $f^*_i \textnormal{grad}\,\Phi_i(p) = 0$ if $w_i(f_i(p),0) \neq 0\}$
\end{itemize}
\end{theorem}

\begin{figure}
    \centering
    \begin{tikzcd}[row sep=20pt, column sep=20pt]
              &                                                                              & M \arrow[d, "f_2"'] \arrow[rd, "f_3"] \arrow[ld, "f_1"'] &                                                     &           \\
              & N_1 \arrow[ld, "{f_{1,1}}"'] \arrow[d, "{f_{1,2}}"'] \arrow[rd, "{f_{1,3}}"] & N_2                                                      & N_3 \arrow[d, "{f_{3,1}}"'] \arrow[rd, "{f_{3,2}}"] &           \\
    {N_{1,1}} & {N_{1,2}}                                                                    & {N_{1,3}}                                                & {N_{3,1}}                                           & {N_{3,2}}
    \end{tikzcd}
    \caption{Tree of manifolds and task maps in a multi-task PBDS policy. At the root is the robot configuration manifold, and task manifolds are at the leaves. This structure can be exploited to parallelize computation of contributions from independent tasks and reuse computation from parent nodes.}
    \label{fig:computational_tree}
\end{figure}
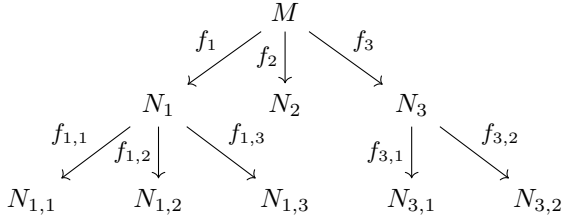


\section{Comparison to Riemannian Motion Policies on Geometric Dynamical Systems}
\label{sec:gds_comparison}
As mentioned, prior work proposed a framework called RMPflow which presented Geometric Dynamical Systems (GDS) for designing dynamical systems on task manifolds that could be ``pulled back" onto the robot manifold to generate a robot motion policy~\cite{ChengMukadamEtAl2018}. Similar to this work, multiple GDSs could be combined through weighted optimization, in that case by using Riemannian Motion Policies (RMP)~\cite{RatliffIssacEtAl2018}. In this section, we will explain the key difficulties with the RMPflow approach (which we will refer to as GDS/RMP) and how they are resolved by our multi-task PBDS policies.\vspace{-2pt}

\subsection{Geometric Consistency}
\label{subsec:geo_consistency}
The main shortcoming of the GDS framework is that it is in general not geometrically consistent. Geometric consistency is an essential property for any geometric method, ensuring its properties hold on non-Euclidean manifolds (i.e. outside a single locally-Euclidean patch) and are invariant to changes in coordinates. To see where geometric consistency is lost, we can derive the GDS framework as a modification of a PBDS applied to tangent bundles as follows.

For a given task map $f : M \longrightarrow N$, we define a higher-order task map $F : TM \longrightarrow TN$ as in Sec.~\ref{sec:pbds}. Since tangent bundles also have a manifold structure, we can build a PBDS on $F$ using the process in Sec.~\ref{sec:pbds}, now operating on the higher-order manifolds shown in Fig.~\ref{fig:higher_order_pbds}. However, now it is difficult to extract a robot acceleration policy corresponding to a curve on the robot manifold $M$ because the PBDS curves $\sigma$ have moved to $TM$. In particular, $\sigma(t) \in TM$ now contains both robot positions and velocities, so $\ddot \sigma(t)$ contains both acceleration and jerk. The GDS tackles this problem by projecting the geometric acceleration $F^*D_\sigma \gamma = a^i \partial^v_i + \kappa^i \partial^a_i$ onto the first half of the basis vectors $\partial^v_i$ to remove components that would correspond to jerk. Unfortunately, this is a projection onto a particular choice of horizontal bundle, which can only be made disjoint from the globally well-defined vertical bundle given a choice of chart \cite{Saunders1989}. In other words, it is impossible to define this projection globally.

In particular, a clear problem arises when considering a change of coordinates on the projected Riemannian task metric for the GDS. Consider a metric $g$ on $TN$ which can be represented in chart $\tilde C$ as $\tilde g^v_{ij} d \tilde v^i d \tilde v^j + \tilde g^a_{ij} d \tilde a^i d \tilde a^j$ so that $\tilde{\bm g} = \textrm{blockdiag}(\tilde{\bm g}^v, \tilde{\bm g}^a)$. Let $\phi$ be the transition function from a new chart $\hat C$ to $\tilde C$. The Jacobian of $\phi$ has the form
\begin{equation}
\begin{gathered}
\bm{J\phi}(\hat{p} , \hat{v}) = \begin{bmatrix}
\bm{J\phi}^{v}(\hat{p} , \hat{v}) & 0 \\
\bm{J\phi}^{av}(\hat{p} , \hat{v}) & \bm{J\phi}^{a}(\hat{p} , \hat{v}) 
\end{bmatrix},\\
J\phi^{v}_{ij}(\hat{p} , \hat{v}) = J\phi^{a}_{ij}(\hat{p} , \hat{v}) = \frac{\partial \tilde p^j}{\partial \hat p^i}(\hat p),\\\vspace{-4pt}
J\phi^{av}_{ij}(\hat{p} , \hat{v}) = \hat v^k   \frac{\partial^2 \tilde p^j}{\partial \hat p^i \partial \hat p^k}(\hat p).
\end{gathered}
\end{equation}
\begin{figure}
    \centering
\begin{tikzcd}[row sep=25pt, column sep=large]
TTM \arrow[r, "F^*dF"] \arrow[rd, "\pi_{TM}"'] & F^*TTN \arrow[d, "F^*\pi_{TN}"] \arrow[r, "F^*\textrm{Id}_{TTN}"] & TTN \arrow[d, "\pi_{TN}"] \\
                                                     & TM \arrow[r, "F"] \arrow[d, "\pi_M"']              & TN \arrow[d, "\pi_N"]     \\
                                                     & M \arrow[r, "f"]                                   & N                        
\end{tikzcd}
    \caption{Commutative diagram for a PBDS system applied to a task map between tangent bundles.}
    \label{fig:higher_order_pbds}
\end{figure}
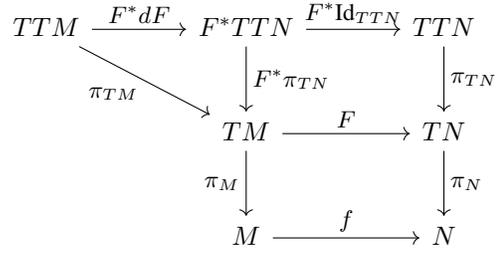
The metric coordinates in $\hat C$ can then be found by
\begin{align}
\hat{\bm g} &= \bm{J\phi}^\top \tilde{\bm g}\bm{J\phi} \\
&= \begin{bmatrix}
\bm{J\phi}^{v\top} \tilde{\bm g}^v \bm{J\phi}^v + \bm{J\phi}^{av\top} \tilde{\bm g}^a \bm{J\phi}^{av} & \bm{J\phi}^{av\top} \tilde{\bm g}^a \bm{J\phi}^{a}\\
\bm{J\phi}^{a\top} \tilde{\bm g}^a \bm{J\phi}^{av} & \bm{J\phi}^{a\top} \tilde{\bm g}^a \bm{J\phi}^{a}\nonumber
\end{bmatrix}
\end{align}
so that $\hat{\bm g}^v = \bm{J\phi}^{v\top} \tilde{\bm g}^v \bm{J\phi}^v + \bm{J\phi}^{av\top} \tilde{\bm g}^a \bm{J\phi}^{av}$. Since $\tilde{\bm g}^a \neq \bm 0$ by the properties of a metric, this shows that a horizontal projection onto $g^v$ cannot be maintained through coordinate transformations. Since the projected geometric acceleration depends on $g^v$, the policy will change depending on the choice of coordinates and will thus exhibit unnatural behavior (e.g., undesirable geodesics) on non-Euclidean manifolds. 

To make this concrete, consider our Table \ref{tab:example} example on a sphere, omitting obstacle avoidance.
To see the effect of coordinate choice, we run both PBDS and GDS policies using different schemes for choosing coordinate charts on $\mathbb{S}^2$. The resulting trajectories can be see in Fig.~\ref{fig:geo_consistency}a and b. Here, the suboptimal and inconsistent behavior of the GDS policy demonstrates the importance of geometric consistency. With such consistency, the policy can accurately capture and leverage the geometry of the manifold, as shown in the PBDS example. Without it, the policy instead introduces artifacts that disturb natural motion on the manifold and can lead to erratic behavior. In some cases it is possible for the GDS policy designer to engineer tasks within a given coordinate choice that suppresses such behavior, but for non-Euclidean metrics this will often be a struggle against the locally-defined curvature terms introduced by the GDS formalism.\vspace{-1pt}

\subsection{Task Prioritization and Velocity-Dependent Metrics}
\label{subsec:task_priorities_vel_dep_metrics}
The main motivation for building GDSs on tangent bundles is to define velocity-dependent metrics on task tangent bundles which can be used both to drive the GDS dynamics and enable velocity-dependent prioritization of tasks in the RMP framework. However, this introduces difficulty for the designer in forming task metrics which both give correct single-task behavior and weight each task correctly against all other tasks. Indeed, rather than allowing incremental design and modular combination of tasks, it can require jointly redesigning tasks for every new task combination.

To resolve this issue, the PBDS framework leverages two observations: (O1) There are few practical scenarios where a velocity-dependent Riemannian metric is useful for driving task dynamics, and (O2) there is no need to use a single metric for both task dynamics and task prioritization. 

We can gain intuition for (O1) by considering a modified PBDS applied to tangent bundle tasks. For  brevity, we outline the main modifications here and leave further details to \iftoggle{ext}{Appx.~\ref{appxsec:tangent_bundle_pbds}}{the Appendix~\cite{BylardBonalliEtAl2021}}. Let $F_i : TM \longrightarrow TN_i$ be our task maps. To avoid projecting onto the horizontal bundle as in the GDS formulation, we instead wish to project onto the vertical bundle, an operation that is globally well-defined. However, to continue recovering an acceleration policy rather than a jerk policy, we consider permuted task maps $\hat F_i \triangleq \sigma_{TN_i} \circ F_i : TM \longrightarrow \mathbb{R}^{2d_i}$ defined locally as $\hat F_i(p,v) = ((d\bar f_i)_p(v),\bar f_i(p))$. This essentially swaps the position and velocity coordinates, an operation which is globally well-defined on parallelizable manifolds (e.g., Lie groups) having an embedding $\bar \varphi_i : N_i \longrightarrow \mathbb{R}^{d_i}$, giving $\bar f_i = \bar \varphi_i \circ f_i$. Then by considering curves in the corresponding pullback bundles satisfying a system analogous to (\ref{eq:multi_pbds_def}) with the appropriate global projections applied, we arrive at the following ODE for a curve $\sigma$ in configuration manifold $M$:\vspace{-3pt}
\begin{equation*}
\begin{gathered}
    \textstyle\ddot{\bm\sigma} = \left(\sum_i\bm{\mathcal C}_i^\top \bm{w}_i^a \bm{\mathcal C}_i\right)^\dagger \left(\sum_i \bm{\mathcal C}_i^\top\bm{w}_i^a\bm{\mathcal{A}}_i\right), \quad \bm{\mathcal C}_i = \bm J \bar{\bm f}_i + \bm \Xi_i^v \vspace{-2pt} \\
    \bm{\mathcal A}_i = (\bm g_i^{-1})^a (\bm{\mathcal F}_{D,i} - \nabla \bm\Phi_i) - (\dot{\bm J \bar{\bm f}}_i  - \bm \Xi_i^a )\dot{\bm \sigma} \\
    \Xi^v_{ij} = J\hat F_{kj} \Gamma^{i+d}_{k(h+m)} J\bar f_{h\ell} \dot\sigma^\ell, \; \Xi^a_{ij} = J\hat F_{k(j+d)} \Gamma^{i+d}_{k(h+m)} J\bar f_{h\ell} \dot\sigma^\ell.
\end{gathered}
\end{equation*}

It is now instructive to see the corresponding local policy equations for a set of uniformly weighted $f_i = \textrm{Id}_{\mathbb R}$ tasks:
\begin{equation*}
\begin{gathered}
\ddot\sigma = \frac{\sum_i (1 + \Xi_i^a)((g^a_i)^{-1}(\mathcal{F}_{D,i} - \partial_{x_i}\Phi_i) - \Xi^v_i\dot\sigma)}{\sum_i (1 + \Xi^a_i)^2} \\
\textstyle\Xi^v_i = \frac{1}{2}(g^a_i)^{-1} \partial_{x_i} g_i^a \dot\sigma, \quad%
\Xi^a_i = \frac{1}{2}(g^a_i)^{-1} \partial_{v_i} g_i^a \dot\sigma.
\end{gathered}
\end{equation*}
This shows that a velocity-dependent metric modifies both the ``force" of the task (its contribution to the numerator of the policy equation) and the total ``mass" of the system (the denominator). However, while modifying a task force is desirable, modifying the total mass is not.

For example, consider using a velocity-dependent metric to enforce constraints. For position constraints, if a task metric increases as the robot approaches the constraint with velocity towards it, the total mass increases. This indeed makes it hard for the robot to accelerate further towards the constraint, but likewise it becomes hard for the robot to accelerate away. Similarly for velocity constraints, a barrier-type velocity metric can enforce velocity limits, but it then becomes difficult for other tasks to remove the high inertia of the robot near the velocity limit.

Thus rather than applying the PBDS framework to tasks mapping from the robot configuration tangent bundle $TM$, we instead nominally consider tasks mapping from the base robot manifold $M$ (i.e., using only position-dependent metrics to design individual task behaviors), greatly simplifying the framework with little practical loss in expressiveness of tasks. Additionally, leveraging (O2), we lose no expressiveness in defining velocity-dependent prioritization of tasks by defining separate pseudometrics on the $TN_i$ to perform the weighting. This also removes the difficulty of achieving desired task behavior and task prioritization using a single metric, thus resulting in a more modular framework.
\subsection{Metric-based Constraints}
\label{subsec:constraints}
As mentioned, constraints imposed solely by Riemannian metrics rather than repulsive potentials can eliminate the spurious local minima in combined potential fields characteristic of many APF methods. However, this has not been considered in works employing the GDS/RMP framework~\cite{LiChengEtAl2019,MukadamChengEtAl2020,RanaLiEtAl2020}. The notion of constraint-enforcing Riemannian metrics was recently considered in~\cite{MainpriceRatliffEtAl2020}. However, rather than constructing such metrics explicitly, the approach numerically computes a distance field from the goal considering obstacles, interpreting this as a geodesic distance field whose gradient is then used to guide trajectory optimization. Such a technique is computationally intensive and may not be practical for complex, dynamic scenarios.

Instead, we propose a class of simple analytical Riemannian behavior metrics that provide tunable enforcement of constraints within the PBDS framework. To motivate the form of these metrics, consider a set of $f_i : \mathbb{R} \longrightarrow \mathbb{R}$ tasks using the PBDS motion policy of (\ref{eq:local_pbds_policy}):
$$
\ddot\sigma = \frac{\sum_i Jf_i w^a_i((g_i)^{-1}(\mathcal{F}_{D,i} - \partial_{x_i}\Phi_i) - (\dot{Jf}_i + \Xi_i)\dot\sigma)}{\sum_i Jf_i^2 w^a_i}.
$$
For constraint tasks, we let $f_i$ be a distance function to the constraint boundary so that $x_i \in \mathbb{R}_+$ is the distance to the constraint and $\Xi_i = \frac{1}{2}(g_i)^{-1} \partial_{x_i} g_i \dot\sigma$. Since we add no repulsive potential and offload dissipative forces to other tasks, constraints are enforced through $\Xi_i$. In short, designing $\Xi_i$ as a negative barrier function at the constraint boundary will slow negative velocities as desired. However, due to the requirements of $g_i$ and the form of $\Xi_i$, a simple logarithmic or inverse barrier $g_i$ is insufficient. In particular, $g_i = \log x_i$ is negative for $x_i < 1$, and $g_i = a/x_i^b$ always results in $\Xi_i = -b/x_i$ for all $a,b\in\mathbb{R}$, which has minimal tuning capability (i.e. $a/x_i^b$ cannot decrease fast enough with decreasing $x_i$). There are many suitable options, but a good candidate is
\begin{equation}
\label{eq:constraint_metric}
g_i(x_i) = \exp(a/(bx_i^b))
\end{equation}
for $a > 0$, $b > 1$, which results in a familiar and flexible barrier function $\Xi_i = -a/x^{b-1}_i$.

Additionally, we must weight constraint tasks such that they are only active when $\dot x_i < 0$, both to save computation and because such metric-based constraints can produce an undesired acceleration away from the constraint when \mbox{$\dot x_i > 0$}. This can be accomplished using a pseudometric
\begin{equation}
\label{eq:toggled_pseudometric}
w^a_i(x_i, \dot x_i) = \begin{cases} 1 &\mbox{if } \dot x_i > 0 \mbox{ and } x_i < \beta,\\
0 &\mbox{otherwise,}
\end{cases}
\end{equation}
where $\beta \in \mathbb{R}_+$ activates the constraint avoidance within some proximity. This works well numerically, but smooth approximations can be used to retain smoothness guarantees.

The effectiveness of this metric and pseudometric combination can be seen in our full Table \ref{tab:example} running example with results shown in Fig.~\ref{fig:geo_consistency}\subref{fig:pbds_sphere_obstacles}, where each obstacle has a task with a task map $f_i$ giving the Euclidean distance to the obstacle in the ambient $\mathbb{R}^3$, metric $g_i = \exp(1/x_i^2)$, and weighting pseudometric $w_i$ as defined in (\ref{eq:toggled_pseudometric}) with $\beta = \infty$.

\begin{figure*}[t!]
    \centering
    \begin{subfigure}[t]{0.32\textwidth}
        \centering
        \includegraphics[width=2.2in]{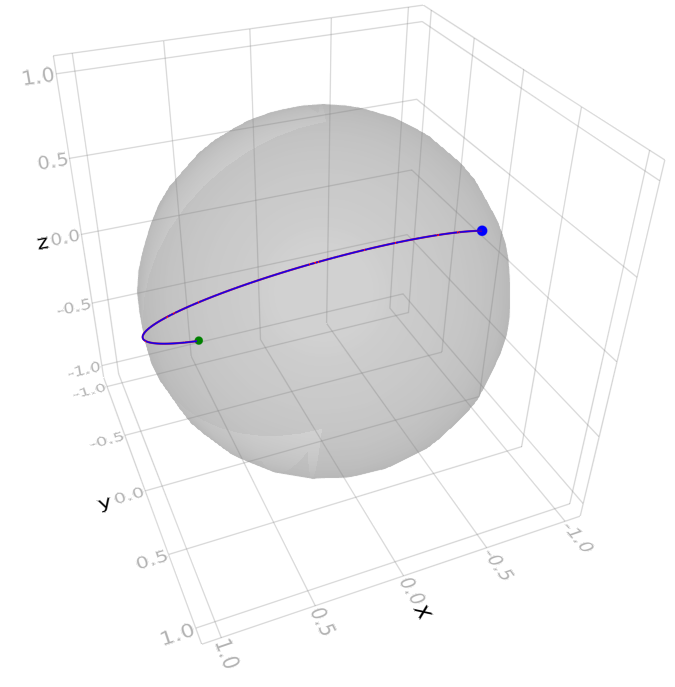}%
        \caption{PBDS geometric consistency test}
        \label{fig:geo_consistency_pbds}%
    \end{subfigure}
    \begin{subfigure}[t]{0.32\textwidth}
        \centering
        \includegraphics[width=2.2in]{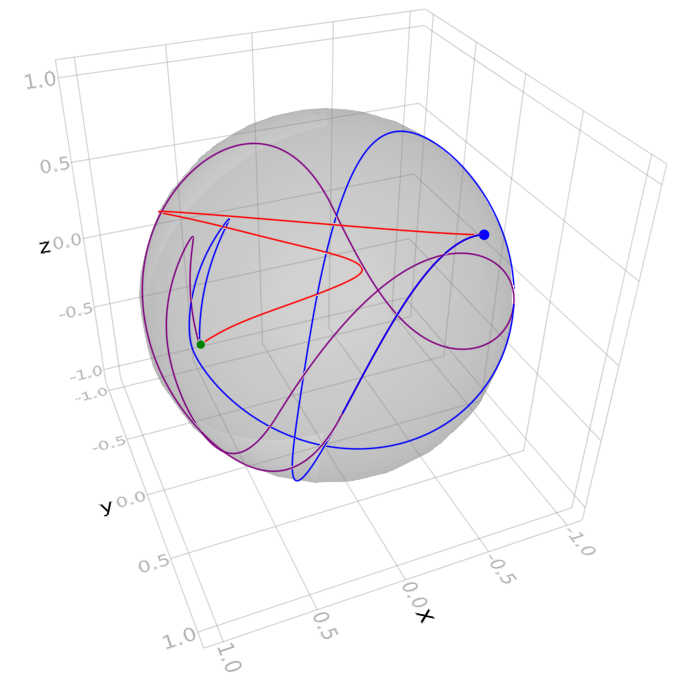}%
        \caption{GDS/RMP geometric consistency test}
        \label{fig:geo_consistency_gds}%
    \end{subfigure}
    \begin{subfigure}[t]{0.32\textwidth}
        \centering
        \includegraphics[width=2.2in]{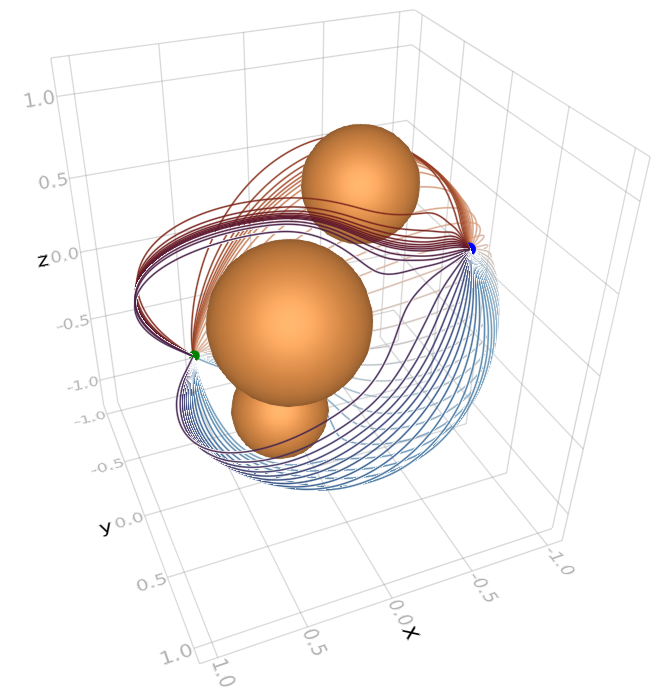}%
        \caption{PBDS attractor and obstacle avoidance policy on $\mathbb{S}^2$}
        \label{fig:pbds_sphere_obstacles}%
    \end{subfigure}
  
    \caption{(a) and (b) show a geometric consistency test using a point attractor on $\mathbb{S}^2$ and three schemes for choice of chart: (Red) Always using south pole stereographic projection, (Blue) always using north pole stereographic projection, (Purple) switching stereographic projections depending on the sphere hemisphere. The blue and green dots show the start and goal points, respectively. Note that using PBDS, all coordinate choices result in the same trajectory. (c) shows resulting trajectories from a PBDS attractor policy with obstacles added, where the curve colors represent different starting velocity directions.}
    \label{fig:geo_consistency}
    \vspace{-15pt}
\end{figure*}

\section{Implementation and Arm Experiments}
\label{sec:experiments}
We implemented the PBDS framework and the GDS/RMP framework for comparison in a fast Julia package called PBDS.jl\footnote{Available at \url{https://github.com/StanfordASL/PBDS.jl}}. Despite the apparent complexity of its geometric formulation, the PBDS algorithm in implementation is quite simple and is summarized in Alg. \ref{algo:pbds}. The main challenge in handling non-Euclidean manifolds within the PBDS framework is transitioning coordinate representations of the necessary objects between different coordinate charts and embeddings. However, PBDS.jl is designed to handle these transitions automatically. To increase performance in complex tasks, PBDS.jl also implements a computational tree inspired by~\cite{ChengMukadamEtAl2018} for cases where a task map tree structure such as in Fig.~\ref{fig:computational_tree} allows significant reuse of computation. This requires careful segmenting  and recombination of the main equations (\ref{eq:local_pbds_policy}) and (\ref{eq:A}) (details in \iftoggle{ext}{Appx.~\ref{appxsec:comp_tree}}{the Appendix~\cite{BylardBonalliEtAl2021}}).

\begin{algorithm}[t!]
\caption{Multi-Task PBDS Policy}\label{algo:pbds}
\begin{algorithmic}[1]
\State {\bf Input: }{Robot position and velocity $(\sigma,\dot\sigma)$}
\State {\bf Output: }{Robot acceleration command $\ddot\sigma$}
\State {\bf Data: }{Task PBDSs $(f_i,g_i,\Phi_i,\mathcal F_{D,i})$, weights $w_i$}
\State Compute $\mathcal A_i$ using (\ref{eq:A}) for each task
\State Compute combined acceleration policy $\ddot\sigma$ using (\ref{eq:local_pbds_policy})
\end{algorithmic}
\end{algorithm}

For demonstration of PBDS on a complex robotic task, we considered a 7-DoF Franka Panda arm grasping a mug in a dynamic, cluttered environment (see video at \href{http://bit.ly/pbds-vid}{bit.ly/pbds-vid}). For mechanism modelling and visualization, we used packages from JuliaRobotics~\cite{KoolenDeits2019}. For the PBDS policy, we used attractor, obstacle avoidance, and joint-limit policies on Euclidean task manifolds and considered damping on $\mathbb S^2$ and $\mathbb S^1$ for gripper location around the mug and along the mug rim. Along with providing natural movement through obstacles, the fast Julia implementation allowed the robot to react quickly to changes in the environment by computing policy outputs at a rapid 300-500 Hz, a speed which could be accelerated further by exploiting the clear opportunities for parallelization offered by the PBDS framework.


\section{Conclusions}
\label{sec:conclusions}
In conclusion, we have provided a fast, easy-to-use, and geometrically consistent framework for generating motion policies on non-Euclidean robot and task spaces. We have also shown how PBDS can leverage Riemannian metrics for simple constraint enforcement without generating undesirable potential function local minima. We highlight that although designing tasks on non-Euclidean manifolds can appear daunting, in practice the design is often trivial (Table \ref{tab:example}), and part of the strength of PBDS lies in its ability to seamlessly integrate different tasks that are designed in isolation. We also stress that a rich, diverse set of complex behaviors can be produced from combinations of a handful of common task design patterns. Future work for PBDS includes extensions to a broader class of constraints (e.g., velocity, acceleration, and control limits) and to tasks defined on discrete structures such as manifold triangle meshes.


\bibliographystyle{IEEEtran}
\bibliography{IEEEabrv,main,ASL_papers,appx}

\newcommand{\noopsort}[1]{} \newcommand{\printfirst}[2]{#1}
  \newcommand{\singleletter}[1]{#1} \newcommand{\switchargs}[2]{#2#1}
\begin{thebibliography}{10}
\providecommand{\url}[1]{#1}
\csname url@rmstyle\endcsname
\providecommand{\newblock}{\relax}
\providecommand{\bibinfo}[2]{#2}
\providecommand\BIBentrySTDinterwordspacing{\spaceskip=0pt\relax}
\providecommand\BIBentryALTinterwordstretchfactor{4}
\providecommand\BIBentryALTinterwordspacing{\spaceskip=\fontdimen2\font plus
\BIBentryALTinterwordstretchfactor\fontdimen3\font minus
  \fontdimen4\font\relax}
\providecommand\BIBforeignlanguage[2]{{%
\expandafter\ifx\csname l@#1\endcsname\relax
\typeout{** WARNING: IEEEtran.bst: No hyphenation pattern has been}%
\typeout{** loaded for the language `#1'. Using the pattern for}%
\typeout{** the default language instead.}%
\else
\language=\csname l@#1\endcsname
\fi
#2}}

\bibitem{Khatib1985}
O.~Khatib, ``Real-time obstacle avoidance for manipulators and mobile robots,''
  in \emph{{Proc.\ IEEE Conf.\ on Robotics and Automation}}, 1985.

\bibitem{KhoslaVolpe1988}
P.~Khosla and R.~Volpe, ``Superquadric artificial potentials for obstacle
  avoidance and approach,'' in \emph{{Proc.\ IEEE Conf.\ on Robotics and
  Automation}}, 1988.

\bibitem{Hogan1984}
N.~Hogan, ``Impedance control: {An} approach to manipulation,'' in
  \emph{{American Control Conference}}, 1984.

\bibitem{Khansari-ZadehBillard2012}
S.~M. Khansari-Zadeh and A.~Billard, ``A dynamical system approach to realtime
  obstacle avoidance,'' \emph{{Autonomous Robots}}, vol.~32, no.~4, pp.
  433--454, 2012.

\bibitem{IjspeertNakanishiEtAl2013}
A.~J. Ijspeert, J.~Nakanishi, H.~Hoffmann, P.~Pastor, and S.~Schaal,
  ``{Dynamical Movement Primitives}: {Learning} attractor models for motor
  behaviors,'' \emph{{Neural Computation}}, vol.~25, no.~2, pp. 328--373, 2013.

\bibitem{RatliffToussaintEtAl2015}
N.~Ratliff, M.~Toussaint, and S.~Schaal, ``Understanding the geometry of
  workspace obstacles in motion optimization,'' in \emph{{Proc.\ IEEE Conf.\ on
  Robotics and Automation}}, 2015.

\bibitem{AgirrebeitiaAvilesEtAl2005}
J.~Agirrebeitia, R.~Avil\'{e}s, I.~F. {de Bustos}, and G.~Ajuria, ``A new {APF}
  strategy for path planning in environments with obstacles,'' \emph{{Mechanism
  and Machine Theory}}, vol.~40, no.~6, pp. 645--658, 2005.

\bibitem{SchulmanHoEtAl2013}
J.~Schulman, J.~Ho, A.~Lee, I.~Awwal, H.~Bradlow, and P.~Abbeel, ``Finding
  locally optimal, collision-free trajectories with sequential convex
  optimization,'' in \emph{{Robotics: Science and Systems}}, 2013.

\bibitem{KuindersmaDeitsEtAl2015}
S.~Kuindersma, R.~Deits, M.~Fallon, A.~Valenzuela, H.~Dai, F.~Permenter,
  T.~Koolen, P.~Marion, and R.~Tedrake, ``Optimization-based locomotion
  planning, estimation, and control design for the {Atlas} humanoid robot,''
  \emph{{Autonomous Robots}}, vol.~40, no.~3, pp. 429--455, 2015.

\bibitem{MerktIvanEtAl2019}
W.~Merkt, V.~Ivan, and S.~Vijaykumar, ``Continuous-time collision avoidance for
  trajectory optimization in dynamic environments,'' in \emph{{IEEE/RSJ Int.\
  Conf.\ on Intelligent Robots \& Systems}}, 2019.

\bibitem{RimonKoditschek1992}
E.~Rimon and D.~E. Koditschek, ``Exact robot navigation wsing artificial
  potentia functions,'' \emph{{IEEE Transactions on Robotics and Automation}},
  vol.~8, no.~5, pp. 501--518, 1992.

\bibitem{KimKhosla1992}
J.-O. Kim and P.~K. Khosla, ``Real-time obstacle avoidance using harmonic
  potential functions,'' \emph{{IEEE Transactions on Robotics and Automation}},
  vol.~8, no.~3, pp. 338--349, 1992.

\bibitem{HuberBillardEtAl2019}
L.~Huber, A.~Billard, and J.-J. Slotine, ``Avoidance of convex and concave
  obstacles with convergence ensured through contraction,'' \emph{{IEEE
  Robotics and Automation Letters}}, vol.~4, no.~2, pp. 1462--1469, 2019.

\bibitem{KorenBorenstein1991}
Y.~Koren and J.~Borenstein, ``Potential field methods and their inherent
  limitations for mobile robot navigation,'' in \emph{{Proc.\ IEEE Conf.\ on
  Robotics and Automation}}, 1991.

\bibitem{Kim2009}
D.~H. Kim, ``Escaping route method for a trap situation in local path
  planning,'' \emph{{Int.\ Journal of Control, Automation and Systems}},
  vol.~7, no.~3, pp. 495--500, 2009.

\bibitem{Lee2018}
J.~M. Lee, \emph{Introduction to Riemannian Manifolds}, 2nd~ed.\hskip 1em plus
  0.5em minus 0.4em\relax {Springer}, 2018.

\bibitem{ChengMukadamEtAl2018}
C.-A. Cheng, M.~Mukadam, J.~Issac, S.~Birchfield, D.~Fox, B.~Boots, and
  N.~Ratliff, ``{RMPflow:} {A} computational graph for automatic motion policy
  generation,'' in \emph{{Workshop on Algorithmic Foundations of Robotics}},
  2018.

\bibitem{MainpriceRatliffEtAl2020}
J.~Mainprice, N.~Ratliff, M.~Toussaint, and S.~Schaal. (2020) An interior point
  method solving motion planning problems with narrow passages. {Available at
  }\url{https://arxiv.org/abs/2007.04842}.

\bibitem{RatliffVanWykEtAl2020}
N.~D. Ratliff, K.~Van~Wyk, M.~Xie, A.~Li, and M.~A. Rana. (2020) Optimization
  fabrics. {Available at }\url{https://arxiv.org/abs/2008.02399}.

\bibitem{RatliffIssacEtAl2018}
N.~D. Ratliff, J.~Issac, D.~Kappler, S.~Birchfield, and D.~Fox. (2018)
  {Riemannian} {Motion} {Policies}. {Available at
  }\url{https://arxiv.org/abs/1601.04037}.

\bibitem{RanaLiEtAl2020}
M.~A. Rana, A.~Li, H.~Ravichandar, M.~Mukadam, S.~Chernova, D.~Fox, B.~Boots,
  and N.~Ratliff, ``Learning reactive motion policies in multiple task spaces
  from human demonstrations,'' in \emph{{Conf.\ on Robot Learning}}, 2020.

\bibitem{Lee2013}
J.~M. Lee, \emph{Introduction to Smooth Manifolds}, 2nd~ed.\hskip 1em plus
  0.5em minus 0.4em\relax {Springer}, 2013.

\bibitem{BulloLewis2004}
F.~Bullo and A.~D. Lewis, \emph{Geometric Control of Mechanical Systems}.\hskip
  1em plus 0.5em minus 0.4em\relax {Springer-Verlag}, 2004.

\bibitem{Saunders1989}
D.~J. Saunders, \emph{The Geometry of Jet Bundles}.\hskip 1em plus 0.5em minus
  0.4em\relax {Cambridge Univ.\ Press}, 1989.

\bibitem{LiChengEtAl2019}
A.~Li, C.-A. Cheng, B.~Boots, and M.~Egerstedt, ``Stable, concurrent controller
  composition for multi-objective robotic tasks,'' in \emph{{Proc.\ IEEE Conf.\
  on Decision and Control}}, 2019.

\bibitem{MukadamChengEtAl2020}
M.~Mukadam, C.-A. Cheng, D.~Fox, B.~Boots, and N.~Ratliff, ``{Riemannian}
  {Motion} {Policy} fusion through learnable {Lyapunov} function reshaping,''
  in \emph{{Conf.\ on Robot Learning}}, 2020.

\bibitem{KoolenDeits2019}
T.~Koolen and R.~Deits, ``Julia for robotics: {Simulation} and real-time
  control in a high-level programming language,'' in \emph{{Proc.\ IEEE Conf.\
  on Robotics and Automation}}, 2019.

\bibitem{Hirsch1976}
M.~W. Hirsch, \emph{Differential Topology}.\hskip 1em plus 0.5em minus
  0.4em\relax {Springer}, 1976.

\end{thebibliography}
\iftoggle{ext}{\makeatletter
\renewcommand{\thesubsection}{\Roman{subsection}}
\renewcommand\thesubsectiondis{\Roman{subsection}.}
\renewcommand{\p@subsection}{\thesection.}
\appendices
\newpage
\section{Further Experimental Details and \\ PBDS Design Strategies}
\label{appxsec:exp_details}
In this section, we provide further details of the design of the PBDS policy driving our robot arm grasping experiments, including details of task design on $\mathbb S^1$ and $\mathbb S^2$, distance-toggled tasks, and the general multi-task PBDS design strategy we employ, which may serve as a useful reference for practitioners.

Since all seven joints of the Franka Panda arm have joint limits, the robot configuration manifold was $M = \mathbb R^7$ (note that for revolute joints without joint limits, the circle $\mathbb S^1$ should be used instead). For collision avoidance, we modelled a conservative collision hull around each link using a set of spheres, as shown in Fig. \ref{fig:arm_sim}.

\subsection{Task Set for Grasping a Mug}
\label{subsec:mug_grasping}
The full list of tasks for this PBDS policy and their corresponding task spaces is as follows.

Damping:
\begin{itemize}
    \item $\mathbb R$ : Robot arm joint damping
    \item $\mathbb S^2 \times \mathbb R$ : Distance-toggled damping of gripper around mug
    \item $\mathbb S^1 \times \mathbb R$ : Distance-toggled damping of gripper along mug rim
\end{itemize}

Constraints:
\begin{itemize}
    \item $\mathbb R_+$ : Joint limits (one for each joint)
    \item $\mathbb R_+$ : Obstacle avoidance (one for each pair of obstacle and link collision hull sphere)
    \item $\mathbb R_+$ : Self-collision avoidance (one for each pair of link collision hull spheres having the potential to collide)
\end{itemize}

Attractors:
\begin{itemize}
    \item $\mathbb R^2$ : Distance-toggled attractor above mug
    \item $\mathbb R^2$ : Distance-toggled gripper axis alignment attractor
    \item $\mathbb R^2$ : Distance-toggled grasping angle alignment attractor
    \item $\mathbb R^2$ : Distance-toggled grasp completion attractor
\end{itemize}
Here $\mathbb R_+$ is the set of positive real numbers, as the task maps for all constraints are distance functions to constraint violation restricted to nonzero distances.

\begin{figure}[t!]
    \centering
    \includegraphics[width=\linewidth]{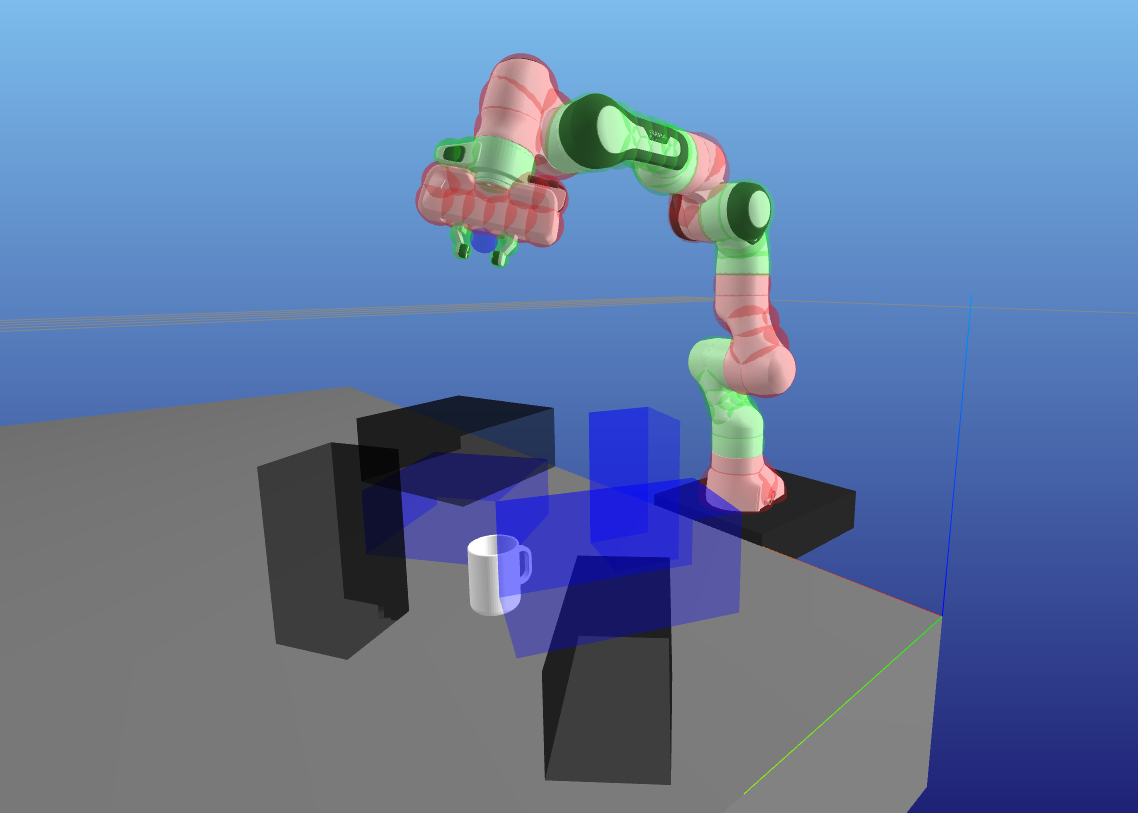}%
    \caption{Example experimental scenario for Franka Panda arm reaching through clutter to grasp a mug. The blue and black boxes are obstacles, and the green and red spheres on the arm represent the collisions for different robot links.}
    \label{fig:arm_sim}%
\end{figure}

To prevent premature grasping and to handle dynamic movement of the mug, we implement a simple two-state machine which uses the above-mug attractor task above while the mug is moving faster than some threshold velocity and until the end-effector nears the attractor goal. Otherwise the other three attractor tasks are used instead. Note that although in this state we have multiple potentials, one for each attractor task, they are easily designed to be always non-conflicting (i.e. the inner product of their resulting desired robot accelerations is never negative), and thus they produce no undesirable local minima.

\subsection{Tasks on $\mathbb S^n$}
\label{subsec:S2_tasks}

For designing tasks on $\mathbb S^n$ for some $n$, it is often natural to consider $\mathbb S^2$ as being embedded in an ambient Euclidean space and design the necessary components $(g, \mathcal F_D, \Phi, w)$ there, as we also do in the example in Table \ref{tab:example}. These components in the ambient space then induce corresponding components on $\mathbb S^n$. Thus, it is important to be able to compute the induced versions of these components. The same reasoning applies to many other manifolds which have a natural embedding in Euclidean space, but here we will consider $\mathbb S^n$ as it is a good example used in both the grasping and Table~\ref{tab:example} scenarios.

The $n$-sphere $\mathbb S^n$ can be naturally embedded as the unit sphere in $\mathbb R^{n+1}$, using an embedding $\bar \varphi : \mathbb S^n \longrightarrow \mathbb R^{n+1}$. We can consider two stereographic projection charts on $\mathbb S^n$ with chart maps $\varphi_+ : U_+ \subset \mathbb S^n \longrightarrow \mathbb R^n$ and $\varphi_- : U_- \subset \mathbb S^n \longrightarrow \mathbb R^n$ representing the south and north pole charts, respectively. Then we can define corresponding maps $\bar \varphi_+$ and $\bar \varphi_-$ from chart coordinates to embedded coordinates as follows:
\begin{equation}
\begin{gathered}
\bar \varphi_+^i(x) = \bar \varphi_-^i(x) = \frac{2x^i}{a}, \ a = \sum_{i=1}^n 1 + (x^i)^2 \\
\bar \varphi_+^{n+1}(x) = -\bar \varphi_-^{n+1}(x) = \frac{2-a}{a}.
\end{gathered}
\end{equation}
There is also a chart transition map between the chart coordinate representations defined by $\varphi_+ \circ \varphi_-^{-1}(x) = \varphi_- \circ \varphi_+^{-1}(x) = x/\lVert x \rVert^2_2$. These maps and the associated spaces are summarized in Fig. \ref{fig:S2}. We can construct similar maps for $T\mathbb S^2$ embedded in $\mathbb R^{2n+2}$.

Now given a behavior metric $\bar g$, dissipative forces $\bar {\mathcal F}_D$, potential function $\bar \Phi$ defined on $\mathbb R^{n+1}$, and weighting pseudometric $\bar w$ defined on $\mathbb R^{2n+2}$, the pullbacks of these objects through $\bar \varphi_+$ and $\bar \varphi_-$ give their coordinate representations in the south and north pole charts, respectively. In particular, we can compute the relevant components in the south pole chart by:
\begin{equation}
\begin{aligned}
    \bm g_+ = \bm J \bar{\bm \varphi}_+^\top \bar{\bm g} \bm J \bar{\bm \varphi}_+, \quad \bm{\mathcal F}_{D+} = \bm J \bar{\bm \varphi}_+^\top \bar{\bm{\mathcal F}}_D\\
    \Phi_+ = \bar{\Phi}\bm\circ \bar{\bm \varphi}_+, \quad \bm w_+^a = \bm J \bar{\varphi}_+^\top \bar{\bm w}^a \bm J \bar{\bm \varphi}_+.
\end{aligned}
\end{equation}
Components in the north pole chart representation can be computed similarly.

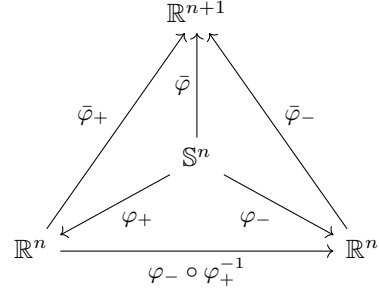
\begin{figure}
    \centering
\begin{tikzcd}[row sep=20pt, column sep=18pt]
                                                                         &  & \mathbb R^{n+1}                                                         &  &                                        \\
                                                                         &  &                                                                         &  &                                        \\
                                                                         &  & \mathbb S^n \arrow[uu, "\bar \varphi"] \arrow[lld, "\varphi_+"] \arrow[rrd, "\varphi_-"'] &  &                                        \\
\mathbb R^n \arrow[rruuu, "\bar \varphi_+"] \arrow[rrrr, "\varphi_-\circ \varphi_+^{-1}"'] &  &                                                                         &  & \mathbb R^n \arrow[lluuu, "\bar \varphi_-"']
\end{tikzcd}
    \caption{Commutative diagram for the two stereographic charts of $\mathbb S^2$ and its embedding in Euclidean space. (Here we omit the subsets $U_+, U_- \subset \mathbb S^2$ which are the domains of the respective chart maps $\varphi_+, \varphi_-$).}
    \label{fig:S2}
\end{figure}

\subsection{Distance-Toggled Tasks}
\label{subsec:distance_toggled_tasks}
In some scenarios it is useful to switch tasks on or off based on some distance using the weighting pseudometrics. We have already seen an example of this in Sec. \ref{subsec:constraints} for constraint tasks, and in the grasping scenario we take advantage of this to skip computation of many of the collision-avoidance tasks at any given configuration. However, as we also exemplify in the grasping example, distance toggles can also be useful for attractor and damping tasks.

To do this, one can simply form a new task map which is the product of the original task map and the distance function used for toggling (if that distance function is different than the function used for the original task map). Thus, the new task manifold becomes the product of the original task manifold and $\mathbb R$, i.e. $\mathbb R$ becomes $\mathbb R^2$ and $\mathbb S$ becomes $\mathbb S^2 \times \mathbb R$. (Note that in PBDS.jl, we have implemented product task maps to conveniently handle product manifolds like $\mathbb S^2 \times \mathbb R$). Given this, one can form weights $w$ for the task which are a function of the new task variable to toggle the task as desired.

\subsection{A General Strategy for Multi-task PBDS Design}
For the practitioner's reference, we summarize here the general strategy we have used for multi-task PBDS policy design in both the Table \ref{tab:example} example and the grasping experiments. Note that this is just one set of options among many potentially viable strategies, and it specifically targets scenarios where the objective is to reach a particular goal configuration or region while avoiding constraint violation. This strategy includes three main components:

1) First, we create a strictly dissipative task over the robot configuration manifold $M$ using the identity task map $f_\textrm{diss} = \textrm{Id}_M$. This task should use a behavior metric corresponding to the desired ``default" trajectories (geodesics) for $M$, e.g. the Euclidean metric for $\mathbb R^n$ and the round metric for $\mathbb S^n$. Likewise, it should also use a suitable ``default" task weighting. In particular, the coordinates of the behavior metric can be reused to specify the operative part of the weighting pseudometric $\bm w^a_\textrm{diss}(p,v) = \bm g_\textrm{diss}(p)$. Lastly, there should be a zero potential function and the force should be strictly dissipative, meaning $\bm {\mathcal F}_{D,\textrm{diss}}(p,v) \cdot \bm v < 0$ for all $\bm v \neq 0$. This task ensures that $(A2)$ and $(A3)$ are always satisfied. In particular, if there is redundancy at any point in the manifold with the respect to the other nonzero-weight tasks, this task serves to regularize the policy.

2) Second, for each constraint, we add a constraint task using a task map $f_\textrm{cons} : M \longrightarrow \mathbb R_+$ that is a smooth distance function to constraint violation. For the behavior metric and weighting pseudometric, we use \eqref{eq:constraint_metric} and \eqref{eq:toggled_pseudometric}, respectively, and we use a zero potential function and zero damping forces.

3) Lastly, to direct the robot toward the goal, we create one or more attractor tasks. The design of these tasks is very flexible. For example, one can add distance-based toggles as described in Sec. \ref{subsec:distance_toggled_tasks}. Additionally, the main portion task map used to define the attractor potential function need not be $\mathbb R$. Both the behavior metric and weighting pseudometrics should use suitable ``defaults" as with the dissipative task, unless toggling is used in the weights. Note also that dissipative forces are optional in this task as the previously defined dissipative task already provides global strict dissipation and is often sufficient for providing aesthetically natural motions. If multiple attractor tasks are used, one should take care that active attractors are non-conflicting, as explained at the end of Sec. \ref{subsec:mug_grasping}.

This strategy for attractors, along with the metric-based enforcement of constraints, ensures that there are no spurious potential function local minima, a common issue with APF approaches. However, it does not mean that there are no spurious (unstable) equilibrium points or that the system is always guaranteed to reach the goal. For more on this, see the discussion following Theorem \ref{appxthe:pbds_stability} in Sec. \ref{appxsubsec:pbds_stability}.


\section{Computational Tree for Multi-Task PBDS}
\label{appxsec:comp_tree}
Here we give a derivation and details for implementation of the PBDS computational tree for reusing computation in multi-level trees of task maps. Consider a set of task maps forming a task map tree such as in Fig. \ref{fig:computational_tree}. In such a case, multiple task maps are composite maps which share components (e.g., $f_{1,1} \circ f_1$ and $f_{1,2} \circ f_1$ share $f_1$). Thus in cases where applying component maps and computing Jacobians and Jacobian derivatives is expensive, it is useful to exploit this tree structure to reuse computation while computing the multi-task PBDS acceleration policy output.

Thus we propose quantities denoted $P$, $A$, $B$, $\mathcal F$, and $\xi$ which can be numerically pulled up the tree from the leaves and combined into equation (\ref{eq:local_pbds_policy}). Noting that each leaf corresponds to a different task, we initialize these quantities for each leaf as follows.

\textbf{Leaf task manifold nodes:}
\begin{equation}
\setlength{\jot}{5pt}
\begin{gathered}
\bm P = \bm B = \bm{Jf}^\top \bm w^a \bm{Jf}, \quad \bm A = \bm{Jf}^\top \bm w^a \dot{\bm{Jf}} \\
\bm{\mathcal F} = \bm{Jf}^\top \bm w^a \bm g^{-1}(\bm{\mathcal F}_D - \nabla\bm\Phi) \\
\xi^q_{sr} = Jf_{\tau q} w^a_{\tau \eta} Jf_{\alpha s} \Gamma^\eta_{\alpha \beta}Jf_{\beta r},
\end{gathered}
\end{equation}
where $f$ is the map to the task manifold from the immediate parent, and we locally compute the task metric $g$, Christoffel symbols $\Gamma$, generalized force $\mathcal F_D$, potential gradient $\nabla\Phi$, and block $\bm w^a$ of the weighting pseudometric $w$ as usual. Next, for parent nodes representing intermediate manifolds in the composite task maps, we combine the child quantities as follows.

\textbf{Intermediate manifold nodes:}
\begin{equation}
\setlength{\jot}{5pt}
\begin{gathered}
\bm P = \bm{Jf}^\top \left(\sum \bm P_c\right) \bm{Jf}, \quad \bm B = \bm{Jf}^\top \left(\sum \bm B_c\right) \bm{Jf} \\
\bm A = \bm{Jf}^\top \left(\left(\sum \bm A_c\right) \bm{Jf} + \left(\sum \bm B_c\right)\dot{\bm{Jf}}\right) \\
\bm{\mathcal F} = \bm{Jf}^\top \sum \bm{\mathcal F}_c, \quad \xi^k_{h \ell} = Jf_{qk} Jf_{sh} \left(\sum \xi_c\right)^q_{sr} Jf_{r\ell},
\end{gathered}\hspace{-10pt}
\end{equation}
where the $c$ subscripts indicating quantities from the child nodes, summations are over all child nodes, and $f$ again denotes the map to the node from the immediate parent. Finally at the root node containing the robot configuration manifold, we combine as follows.

\begin{figure}
    \centering
\begin{tikzcd}[row sep=28pt, column sep=10pt]
          &                                                                  &             & M \arrow[d] \arrow[rrd, "f_b"] \arrow[lld, "f_1"'] &           &                                                                  &             \\
          & N_1 \arrow[ld, "{f_{1,2}}"'] \arrow[d] \arrow[rd, "{f_{1,c_1}}"] &             & \ldots                                             &           & N_b \arrow[ld, "{f_{b,1}}"'] \arrow[d] \arrow[rd, "{f_{b,c_b}}"] &             \\
{N_{1,2}} & \ldots                                                           & {N_{1,c_1}} &                                                    & {N_{b,1}} & \ldots                                                           & {N_{b,c_b}}
\end{tikzcd} 
    \caption{Task map tree for an example multi-task PBDS.}
    \label{fig:comp_tree_proof}
\end{figure}
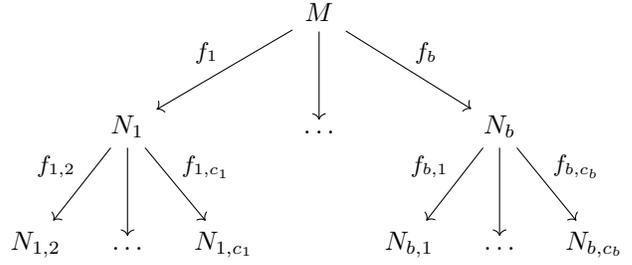

\textbf{Root robot manifold node:}
\begin{equation}
\setlength{\jot}{5pt}
\begin{gathered}
\bm P = \sum \bm P_c, \quad \bm A = \sum \bm A_c, \quad \bm{\mathcal F} = \sum \bm{\mathcal F}_c \\
\xi_{k\ell} = \left(\sum \xi_c\right)^k_{h\ell} \dot\sigma^h.
\end{gathered}
\end{equation}
We can now compute the multi-task PBDS acceleration policy output using the root node quantities as
\begin{equation}
\label{eq:comp_tree_combination}
\ddot{\bm \sigma} = \bm P^\dagger ( \bm{\mathcal F} - (\bm A + \bm \xi) \dot{\bm\sigma}).
\end{equation}
We now demonstrate that this strategy of splitting and reusing computation correctly reconstructs (\ref{eq:local_pbds_policy}).

\begin{lemma}[Computational Tree] Given a multi-task PBDS and a corresponding tree of task maps, computing $P$, $A$, $\mathcal F$, and $\xi$ at the root node gives
\begin{align*}
\bm P^\dagger ( &\bm{\mathcal F} - (\bm A + \bm \xi) \dot{\bm\sigma}) \\
&= \left(\sum_{i}\bm{Jf}_i^\top\bm{w}_i^a\bm{Jf}_i\right)^\dagger\left(\sum_{i} \bm{Jf}_i^\top\bm{w}_i^a\bm{\mathcal{A}}_i\right)
\end{align*}
\end{lemma}
\begin{proof}
Consider the two-level task tree in Fig. \ref{fig:comp_tree_proof}, having $b$ intermediate manifold nodes and $c_i$ children for the $i$th intermediate node, for a total of $\sum_{i=1}^b c_i$ leaf nodes and corresponding task maps. For such a multi-task PBDS, we can expand (\ref{eq:local_pbds_policy}) as
\begin{align*}
\ddot{\bm\sigma} = &\left(\sum_{i=1}^b \sum_{j=1}^{c_i} \bm{Jf}_i^\top \bm{Jf}_{i,j}^\top\bm w_{i,j}^a \bm{Jf}_{i,j} \bm{Jf}_i\right)^\dagger\\
&\Bigg(\sum_{i=1}^b \sum_{j=1}^{c_i} \bm{Jf}_i^\top \bm{Jf}_{i,j}^\top\bm w_{i,j}^a \Bigg(\bm g_{i,j}^{-1}\Big(\bm{\mathcal F}_{D,i,j} - \nabla\Phi_{i,j}\Big)\\
&- \Bigg(\frac{d}{dt}\Big(\bm{Jf}_{i,j}\bm{Jf}_i\Big) + \bm\Xi_{i,j}\Bigg)\dot{\bm\sigma}\Bigg)\Bigg),
\end{align*}
where we use the $i$ subscript to denote quantities associated to the intermediate nodes and use the $i,j$ subscript for those associated to the leaf nodes and their corresponding tasks.

Now unpacking each of the terms in (\ref{eq:comp_tree_combination}) at the root robot configuration manifold node, we have
\begingroup
\allowdisplaybreaks
\begin{flalign*}
\bm P &= \sum_{i=1}^b \bm P_i = \sum_{i=1}^b \bm{Jf}_i^\top \bigg(\sum_{j=1}^{c_i}\bm P_{i,j}\bigg) \bm{Jf}_i&\\
&= \sum_{i=1}^b \sum_{j=1}^{c_i} \bm{Jf}_i^\top \bm{Jf}_{i,j}^\top \bm w_{i,j}^a \bm{Jf}_{i,j} \bm{Jf}_i,&
\end{flalign*}
\begin{align*}
&\bm A = \sum_{i=1}^b \bm A_i\\
&= \sum_{i=1}^b \bm{Jf}_i^\top \Bigg(\bigg(\sum_{j=1}^{c_i} \bm A_{i,j}\bigg)\dot{\bm{Jf}_i} + \bigg(\sum_{j=1}^{c_i} \bm B_{i,j}\bigg)\bm{Jf}_i\Bigg)\\
&= \sum_{i=1}^b \sum_{j=1}^{c_i} \bm{Jf}_i^\top \left(\bm{Jf}_{i,j}^\top \bm w_{i,j}^a \bm{Jf}_{i,j} \dot{\bm{Jf}_i} + \bm{Jf}_i^\top \bm w_{i,j}^a \dot{\bm{Jf}_{i,j}} \bm{Jf}_i\right)\\
&= \sum_{i=1}^b \sum_{j=1}^{c_i} \bm{Jf}_i^\top \bm{Jf}_{i,j}^\top \bm w_{i,j}^a \frac{d}{dt}\Big(\bm{Jf}_{i,j} \bm{Jf}_i\Big),
\end{align*}
\begin{flalign*}
\bm{\mathcal F} &= \sum_{i=1}^b \bm{\mathcal F}_i = \sum_{i=1}^b \bm{Jf}_i^\top \bigg(\sum_{j=1}^{c_i} \mathcal F_{i,j}\bigg)&\\
&= \sum_{i=1}^b \sum_{j=1}^{c_i} \bm{Jf}_i^\top \bm{Jf}_{i,j}^\top \bm w_{i,j}^a \bm g_{i,j}^{-1}\big(\bm{\mathcal F}_{D,i,j} - \nabla \Phi_{i,j}\big),&
\end{flalign*}
\begin{flalign*}
\xi_{k\ell} &= \bigg(\sum_{i=1}^b \xi_i\bigg)^k_{h\ell}\dot{\sigma}^h &\\
&= \sum_{i=1}^b (Jf_i)_{qk} (Jf_i)_{sh} \bigg(\sum_{j=1}^{c_i} \xi_{i,j}\bigg)^q_{sr} (Jf_i)_{r\ell} \dot\sigma^h &\\
&= \sum_{i=1}^b \sum_{j=1}^{c_i} (Jf_i)_{qk} (Jf_{i,j})_{\tau q} (w^a_{ij})_{\tau\eta} (Jf_i)_{sh} (Jf_{i,j})_{\alpha s}&\\
&\quad (\Gamma_{i,j})^\eta_{\alpha\beta} (Jf_{i,j})_{\beta r} (Jf_i)_{r\ell} \dot\sigma^h&\\
&= \sum_{i=1}^b \sum_{j=1}^{c_i} \big(\bm{Jf}_i^\top \bm{Jf}_{i,j}^\top \bm w_{i,j}^a\big)_{k\eta} (Jf_{i,j} Jf_i)_{\alpha h}&\\
&\quad(\Gamma_{i,j})^\eta_{\alpha\beta} (Jf_{i,j} Jf_i)_{\beta\ell} \dot\sigma^h&\\
&= \sum_{i=1}^b \sum_{j=1}^{c_i} \big(\bm{Jf}_i^\top \bm{Jf}_{i,j}^\top \bm w_{i,j}^a\big)_{k\eta} J(f_{i,j} \circ f_i)_{\alpha h}&\\
&\quad(\Gamma_{i,j})^\eta_{\alpha\beta} J(f_{i,j} \circ f_i)_{\beta\ell} \dot\sigma^h &\\
&= \sum_{i=1}^b \sum_{j=1}^{c_i} \big(\bm{Jf}_i^\top \bm{Jf}_{i,j}^\top \bm w_{i,j}^a\big)_{k\eta} (\Xi_{i,j})_{\eta h} \dot\sigma^h &\\
&= \sum_{i=1}^b \sum_{j=1}^{c_i} \bm{Jf}_i^\top \bm{Jf}_{i,j}^\top \bm w_{i,j}^a \bm\Xi_{i,j} \dot{\bm \sigma}.&
\end{flalign*}
Combining these in (\ref{eq:comp_tree_combination}), we see that we indeed recover the correct robot acceleration policy from (\ref{eq:local_pbds_policy}).
\end{proof}
\endgroup


\section{Mathematical Details for Pullback Bundle Dynamical Systems}

\label{appxsec:single_pbds}
In this section, we give a detailed derivation of Pullback Bundle Dynamical Systems and the extraction of multi-task policies, showing that each of the key objects are geometrically well-defined.

Let $M$ and $N$ be smooth $m$- and $n$-dimensional robot configuration and task manifolds, having canonical projections $\pi_M$ and $\pi_N$, respectively. Let $f : M \longrightarrow N$ be a smooth task map and let $\nabla$ be the Levi-Civita connection in $TN$ corresponding to metric $g$ on $N$ with Christoffel symbols
\begin{equation}
\Gamma^k_{ij} = \frac{1}{2}g^{kh}(\partial_i g_{jh} + \partial_j g_{ih} - \partial_h g_{ij}).
\end{equation}
We form a pullback bundle $f^*TN$ as defined in (\ref{eq:pullback_bundle}), which is a smooth $n$-vector bundle over $M$ and is itself an $(m+n)$-dimensional manifold.

\subsection{Constructions on the Pullback Bundle}
\label{subsec:constructions_pb}
Note that since $f^*TN$ is not a tangent bundle, we must reconstruct all of the standard objects and operators normally used for defining mechanical systems on tangent bundles. We first denote the pullback bundle projection as $f^*\pi_N$, and for convenience we define a pullback differential
\begin{equation}
f^*df : TM \longrightarrow f^*TN : (p,v) \mapsto (p, \pi_2(df_p(v)))
\end{equation}
and the pullback of the identity map on $TN$
\begin{equation}
f^*\textrm{Id}_{TN} : f^*TN \longrightarrow TN :(p,v) \mapsto (f(p),v),
\end{equation}
both of which are smooth and globally well-defined.

We now construct a global pullback connection on $f^*TN$:

\begin{lemma}[Pullback Connection]
\label{lem:pullback_connection}
For every $p \in M$, choose Christoffel symbols defined locally as
$$
f^* \Gamma^k_{ij}(p) = \frac{\partial f^\ell}{\partial x^i}(p) \Gamma^k_{\ell j}(f(p)).
$$
These functions are globally extendable and form a global connection (which we call the pullback connection),
$$
f^*\nabla : \Gamma(TM) \times \Gamma(f^*TN) \longrightarrow \Gamma(f^*TN).
$$
\end{lemma}
\begin{proof}
The chosen Christoffel symbols are locally smooth. To show that they are global, we must check whether they satisfy the chart transition formula for Christoffel symbols on vector bundles.

Consider two charts at $p \in M$ denoted by $C_M$ and $\tilde C_M$ and two charts at $f(p) \in N$ denoted by $C_N$ and $\tilde C_N$. Then $TM$ and $TN$ have corresponding local frames $(\frac{\partial}{\partial x^i})$, $(\frac{\partial}{\partial \tilde x^i})$ and $(\frac{\partial}{\partial y^i})$, $(\frac{\partial}{\partial \tilde y^i})$, respectively. Suppose we also have a vector bundle $E$ over $N$ (note that in our case we have $E = TN$). Given local frames $(E_i)$ and $(\tilde E_i)$ for $E$,  there exists a smooth, non-singular matrix of functions $(A_{ij})$ such that locally $\tilde E_i = A_{ij} E_i$.

Now given a connection $\nabla : \Gamma(TN) \times \Gamma(E) \longrightarrow \Gamma(E)$, the related Christoffel symbols can be denoted by $\Gamma^k_{ij}$ and $\tilde\Gamma^k_{ij}$. The chart transition formula for these Christoffel symbols is
\begin{equation}
\tilde\Gamma^k_{ij} = A^{pk} \frac{\partial y^q}{\partial \tilde y^i} \left( A_{jr} \Gamma^p_{qr} + \frac{\partial A_{jp}}{\partial y^q}\right).
\end{equation}
We must now show a similar formula holds for $f^*\Gamma^k_{ij}$. Given the local frames for $E$ above, we can construct local frames for the pullback bundle $f^*E$ as $(f^*E_i) \triangleq (E_i \circ f)$ and $(f^*\tilde E_i) \triangleq (\tilde E_i \circ f)$. Then for $f^*\tilde E_i = B_{ij} f^*E_j$, it is easily seen that $B_{ij} = A_{ij}\circ f$. We can now compute
\begingroup
\allowdisplaybreaks
\begin{align*}
    f^*\tilde\Gamma^k_{ij} &= \frac{\partial\tilde f^\ell}{\partial \tilde x^i}(p) \tilde\Gamma^k_{\ell j}(f(p))\\
    &= \frac{\partial y^q}{\partial \tilde y^\ell}(f(p)) \frac{\partial\tilde f^\ell}{\partial \tilde x^i}(p) A^{pk}(f(p))\\
    &\quad\left(A_{jr}(f(p)) \Gamma^p_{qr}(f(p)) + \frac{\partial A_{jp}}{\partial y^q}(f(p))\right)\\
    &=\frac{\partial x^\ell}{\partial \tilde x^i}(p) \frac{\partial f^q}{\partial x^\ell}(p) B^{pk}(p) \\
    &\quad\left(B_{jr}(p) \Gamma^p_{qr}(f(p)) + \frac{\partial A_{jp}}{\partial y^q}(f(p))\right)\\
    &=\frac{\partial x^\ell}{\partial \tilde x^i}(p) B^{pk}(p)\\
    &\quad\left(B_{jr}(p) \frac{\partial f^q}{\partial x^\ell}(p) \Gamma^p_{qr}(f(p)) + \frac{\partial A_{jp}}{\partial y^q}(f(p)) \frac{\partial f^q}{\partial x^\ell}(p)\right)\\
    &=\frac{\partial x^\ell}{\partial \tilde x^i}(p) B^{pk}(p) \left(B_{jr}(p) f^*\Gamma^p_{\ell r}(p) + \frac{\partial B_{jp}}{\partial x^\ell}(p)\right).
\end{align*}
\endgroup
We therefore find that the Christoffel symbols $f^*\Gamma^k_{ij}$ can be extended globally and the result follows.
\end{proof}

This naturally leads to a geometric acceleration operator for curves with velocity in $f^*TN$. Let $\sigma : I \longrightarrow M$ be a smooth curve, and denote the space of pullback bundle vector fields over $\sigma$ as
\begin{align}
f^* \mathfrak X(\sigma) \triangleq \{V : I \longrightarrow f^*TN \; | \; &V \textrm{ is smooth,}\\
&V(t) \in f^*E_{\sigma(t)}, \forall t\in I\}.\nonumber
\end{align}
As a classical result, for every curve $\sigma$ in $M$, the pullback connection $f^*\nabla$ defines a unique acceleration operator 
$$
f^*D_\sigma: f^* \mathfrak X(\sigma) \longrightarrow f^* \mathfrak X(\sigma),
$$
which locally takes the form
\begin{equation}
    \label{eq:local_pullback_accel}
    f^*D_\sigma V(t) = \left(\dot V^k(t) + \dot \sigma^i(t) V^j(t) f^*\Gamma^k_{ij}(\sigma(t))\right) f^*\partial_k |_{\sigma(t)},
\end{equation}
where $(f^*\partial_i)$ is the frame for $f^*TN$ defined by 
\begin{equation}
f^*\partial_i(p) = (p, \pi_2(\partial_i(f(p))).
\end{equation}

Now we construct a pullback metric on $f^*TN$. Note that Lyapunov-type results for classical mechanical systems rely on the fact that the Levi-Civita connection is \emph{compatible} with the related Riemannian metric. Indeed, this allows one to differentiate the norm of some energy, obtaining such variations as a function of the evolutionary equation of the associated mechanical system. Thus, it is critical to reproduce such a property in our framework to achieve stability.

\begin{lemma} \label{ref_LemmaCompat}
Let $g$ be a Riemannian metric on $N$ which is compatible with connection $\nabla$. Then the pullback metric
\begin{align*}
    f^*g : \; &f^*TN \times f^*TN \longrightarrow \mathbb{R} \\
    &((p,v_1), (p,v_2)) \mapsto g_{f(p)}(v_1,v_2)
\end{align*}
is compatible with the pullback connection $f^*\nabla$.
\end{lemma}
\begin{proof}
Since $g$ and $\nabla$ are compatible, we know they satisfy the compatibility condition
\begin{equation}
    \label{eq:normal_compat_condition}
    \frac{d}{dt} g_{\rho(t)}(X(t), X(t)) = 2g_{\rho(t)}(D_\rho X(t), X(t))
\end{equation}
for any smooth vector field $X$ along any smooth curve $\rho$ in $N$. To show the same for $f^*g$ and $f^*\nabla$, let \mbox{$\sigma : I \longrightarrow M$} be a smooth curve and let $V \in f^*\mathfrak X(\sigma)$. Then \mbox{$\rho : I \longrightarrow N :$} $t \mapsto f(\sigma(t))$ and $W : I \longrightarrow TN : t \mapsto f^*\textrm{Id}_{TN}(V(t))$ are well-defined and smooth. Now using (\ref{eq:normal_compat_condition}) we obtain
\begin{align}
    \label{eq:metric_compat_condition}
    \frac{d}{dt}f^*&g_{\sigma(t)}(V(t), V(t)) \nonumber\vspace{-5pt}\\
    &= \frac{d}{dt}g_{f(\sigma(t))}\left(f^*\textrm{Id}_{TN}(V(t)), f^*\textrm{Id}_{TN}(V(t))\right) \nonumber\\
    &= \frac{d}{dt}g_{\rho(t)}\left(W(t),W(t)\right) \nonumber\\
    &= 2 g_{\rho(t)}\left(D_\rho W(t), W(t)\right).
\end{align}
Applying (\ref{eq:local_pullback_accel}) adapted for $D_\sigma$ gives
\begin{align*}
    D_\rho &W(t) = \left(\dot W^k(t) + \dot\rho^\ell(t) W^j(t) \Gamma^k_{\ell j}(\rho(t))  \right) \partial_k |_{\rho(t)}\vspace{-5pt}\\
    &= \left(\dot W^k(t) + \frac{\partial f^\ell}{\partial x^i}(\sigma(t)) \dot\sigma^i(t) W^j(t) \Gamma^k_{\ell j}(\rho(t))  \right) \partial_k |_{\rho(t)}\vspace{-5pt}\\
    &= \left(\dot W^k(t) + \dot\sigma^i(t) W^j(t) f^*\Gamma^k_{ij}(\sigma(t))  \right) \partial_k |_{\rho(t)}\vspace{-5pt}\\
    &= f^*\textrm{Id}_{TN}\left(\left(\dot V^k(t) + \dot\sigma^i(t) V^j(t) f^*\Gamma^k_{ij}(\sigma(t)) \right)  \partial_k |_{\sigma
    (t)}\right) \\
    &= f^*\textrm{Id}_{TN}(f^* D_\sigma V(t)).
\end{align*}
Thus continuing (\ref{eq:metric_compat_condition}) finally provides \vspace{-3pt}
\begin{flalign*}
    &\qquad\quad\frac{d}{dt}f^*g_{\sigma(t)}(V(t), V(t))&&
\end{flalign*}
\begin{align*}
    &= 2 g_{\rho(t)}\left(f^*\textrm{Id}_{TN}(f^* D_\sigma V(t)), f^*\textrm{Id}_{TN}(V(t))\right) \\
    &= 2 f^*g_{\sigma(t)}(f^* D_\sigma V(t), V(t)). \qedhere
\end{align*}
\end{proof}

Next we develop the sharp operator on $f^*TN$, useful for handling forces. First we define the dual pullback bundle
\begin{equation}
f^*T^*N = \coprod_{p \in M} (\pi^*_N)^{-1}(f(p)),
\end{equation}
where $\pi^*_N$ is the canonical projection for the cotangent bundle $T^*N$. We can now define the flat operator
\begin{equation}
\flat : f^*TN \longrightarrow f^*T^*N : (p,v) \mapsto f^*g_p(v,\cdot),
\end{equation}
which is easily seen to be a smooth bundle isomorphism. The sharp operator is naturally defined to be the inverse of the flat operator, i.e.,
\begin{equation}
\sharp : f^*T^*N \longrightarrow f^*TN : (p,\omega) \mapsto (p,\omega)^\flat.
\end{equation}
In particular, locally we have $(\omega(p)^\sharp)^i = g^{ij}(f(p)) \omega_j(p).$

Now consider dissipative forces $\mathcal F_D : TN \longrightarrow T^*N$, which we define as satisfying the dissipative property $\langle \mathcal F_D(p,v), v\rangle \leq 0$ for all $(p,v) \in TN$, where $\langle \cdot, \cdot \rangle$ is the natural pairing between a covector and a vector. For convenience, we associate to $\mathcal F_D$ corresponding pullback dissipative forces as
\begin{equation}
\begin{aligned}
f^*\mathcal F_D : f^*TN &\longrightarrow f^*T^*N\\
(p,v) &\mapsto (p,\pi_2(\mathcal F_D(f(p),v)),
\end{aligned}
\end{equation}
such that we have $f^*\mathcal F_D(\cdot)^\sharp : f^*TN \longrightarrow f^*TN$. From this definition, it follows that we have $\langle f^*\mathcal F_D(p,v), v\rangle \leq 0$ for all $(p,v) \in f^*TN$ so that these pullback forces inherit the dissipative property. Note that the dissipative forces $\mathcal F_D$ can also be time-dependent and similar results to those shown in this work will hold, but we keep them time-independent for simplicity.

The last object we must construct on $f^*TN$ is a gradient operator for handling potentials. Given any smooth function $\Phi : N \longrightarrow \mathbb R$, we locally define the smooth map
$$
B_\Phi : M \longrightarrow f^*T^*N : p \mapsto \frac{\partial \Phi}{\partial y^i}(f(p))f^*dy^i |_{p},
$$
where $(f^*dy^i)$ is the coframe for $f^*T^*N$ defined by $f^*dy^i(p) = (p, \pi_2(dy^i(f(p)))$. This map is intentionally designed such that for every $(p,v) \in TM$ we have
\begin{align}
    B_\Phi(p)&(f^*df_p(v)) = \frac{\partial \Phi}{\partial y^i}(f(p)) f^*dy^i |_{p}(f^*df_p(v)))\nonumber\\
    &=\frac{\partial \Phi}{\partial y^i}(f(p)) f^*dy^i |_{p}\left(v^j\frac{\partial f^k}{\partial x^j}(p) \left.\frac{\partial}{\partial y^k} \right\vert_{f(p)}\right)\nonumber\\
    &= v^j \frac{\partial f^i}{\partial x^j}(p) \frac{\partial \Phi}{\partial y^i}(f(p)) = v^j \frac{\partial (\Phi \circ f)}{\partial y^i}(p)\nonumber\\
    \label{eq:B_Phi}
    &= d(\Phi \circ f)(p,v),
\end{align}
which is crucial for our stability analysis. Now we can define the pullback gradient operator such that
\begin{equation}
f^*\textrm{grad}\,\Phi : M \longrightarrow f^*TN : p \mapsto B_\Phi(p)^\sharp.
\end{equation}
To denote the total acceleration contributed by the pullback forces (using the sharp operator to indicate converting from forces to accelerations), we can also define the shorthands
\begin{equation}
\begin{aligned}
f^*\mathcal F(\cdot)^\sharp : f^*TN &\longrightarrow f^*TN\\
(p,v) &\mapsto f^*\mathcal F_D(p,v)^\sharp - f^*\textrm{grad}\,\Phi(p).
\end{aligned}
\end{equation}


\subsection{Local Pullback Bundle Dynamical Systems}
\label{appxsubsec:local_pbds}
We finally come to the definition of an important building block for Pullback Bundle Dynamical Systems (PBDS) called a local PBDS.
\begin{definition}[Local Pullback Bundle Dynamical System] Let $f : M \longrightarrow N$ be a smooth task map for a Riemannian task manifold $(N,g)$, let $\Phi : N \longrightarrow \mathbb R_+$ be a smooth potential function, and let $\mathcal F_D : TN \longrightarrow T^*N$ be dissipative forces. Then for each $(p,v)\in M$, we can choose a curve $\alpha_{(p,v)} : (-\varepsilon, \varepsilon) \longrightarrow M$ resulting in $\gamma_{\alpha_{(p,v)}} : (-\varepsilon, \varepsilon) \longrightarrow f^*TN$ for some $\varepsilon > 0$ such that $(f,g,\Phi,\mathcal F_D, \alpha_{(p,v)})$ forms a local Pullback Bundle Dynamical System (PBDS) satisfying 
\begin{equation*}
\textrm{PBDS}_{\alpha_{(p,v)}} \hspace{-5pt}\ \begin{cases}
    f^* D_{\alpha_{(p,v)}} \gamma_{\alpha_{(p,v)}}(s) = f^*\mathcal{F}(\gamma_{\alpha_{(p,v)}}(s))^\sharp\\
    \gamma_{\alpha_{(p,v)}}(0) = f^*df_{\alpha_{(p,v)}(0)}(\alpha_{(p,v)}'(0)),\\ \alpha_{(p,v)}'(0) = (p,v).
\end{cases}
\end{equation*}
\end{definition}

The reason for this local PBDS definition is to ensure that $\textrm{PBDS}$ is well-posed at $t = 0$ (i.e. where a curve $\alpha_{\sigma'(0)}$ is required to define $f^*D_{\alpha_{\sigma'(0)}}$) and in cases where the curve $\alpha_{\sigma'(t)}$ defining a valid $\textrm{PBDS}_{\alpha_{\sigma'(t)}}$ for some $t \in [0, \infty)$ is not unique. For example, the latter may occur when there is redundancy, i.e., $m > n$, which is often the case in practice.

Now we must show that a local PBDS exists at each point in $TM$. In particular, we must show that for any $(p,v) \in TM$, there exists a curve $\alpha_{(p,v)}$ satisfying $\textrm{PBDS}_{\alpha_{(p,v)}}$.  If $(U,\varphi)$ is a local chart for $M$ centered at $p$, let $\alpha_{(p,v)}(s) \triangleq \varphi^{-1}(s\bm v)$, using local coordinates $\bm v \in \mathbb R^m$. This curve is well-defined and smooth for small times around zero, so the desired curve $\alpha_{(p,v)}$ always exists, meaning the corresponding local PBDS always exists.

To begin the process of extracting a robot motion policy, we can define a map giving the desired pullback task accelerations from local PBDSs corresponding to each point in the robot configuration tangent bundle $TM$:
\begin{equation}
\label{eq:appx_pullback_accelerations}
G : TM \longrightarrow T(f^*TN) : (p,v) \mapsto \dot\gamma_{\alpha_{(p,v)}}(0).
\end{equation}
For this map to be well-defined, we must show that it does not depend on the choice of curves $\alpha_{(p,v)}$. Let $\alpha_{(p,v)}$ and $\beta_{(p,v)}$ be two curves in $M$ satisfying $\textrm{PBDS}_{\alpha_{(p,v)}}$ and $\textrm{PBDS}_{\beta_{(p,v)}}$, respectively. These result in corresponding curves $\gamma_{\alpha_{(p,v)}}$ and $\gamma_{\beta_{(p,v)}}$ in $f^*TN$ as provided in the local PBDS definition, which exist as a consequence of standard Cauchy-Lipschitz arguments. In particular, since every quantity defining a local PBDS is smooth, these curves are at least $C^1$, allowing us to compute $\dot\gamma_{\alpha_{(p,v)}}$ and $\dot\gamma_{\beta_{(p,v)}}$ pointwise. Thus using (\ref{eq:local_pullback_accel}) we can compute
\begin{align*}
    \dot\gamma_{\alpha_{(p,v)}}^k(0) =\; &(f^*\mathcal F(\gamma_{\alpha_{(p,v)}}(0))^\sharp)^k\\
    &- \dot\alpha_{(p,v)}^i(0) \gamma_{\alpha_{(p,v)}}^j(0) f^*\Gamma^k_{ij}(\alpha_{(p,v)}(0))\\
    =\; &(f^*\mathcal F(f^*df(\dot\alpha_{(p,v)}(0)))^\sharp)^k\\
    &- \dot\alpha_{(p,v)}^i(0) (f^*df(\dot\alpha_{(p,v)}(0)))^j f^*\Gamma^k_{ij}(\alpha_{(p,v)}(0))\\
    =\; &(f^*\mathcal F(\gamma_{\beta_{(p,v)}}(0))^\sharp)^k\\
    &- \dot\beta_{(p,v)}^i(0) \gamma_{\beta_{(p,v)}}^j(0) f^*\Gamma^k_{ij}(\beta_{(p,v)}(0))\\
    =\; &\dot\gamma_{\beta_{(p,v)}}^k(0).
\end{align*}
Therefore we have $\dot\gamma_{\alpha_{(p,v)}}(0) = \dot\gamma_{\beta_{(p,v)}}(0)$ by the chart invariance of the derivative operator. This shows that $G$ is well-defined and, in particular, smooth.

Given this result, we will adopt the shorthand $\gamma'_{v_p}(0) \triangleq \gamma'_{\alpha_{(p,v)}}(0)$ for $(p,v) \in TM$, knowing that suitable choices of $\alpha_{(p,v)}$ exist.

\subsection{Multi-Task Pullback Bundle Dynamical Systems}
\label{appxsubsec:multi_pbds}
Next, we continue with a detailed derivation for the extraction of multi-task PBDS policies. For $K$ tasks, consider task maps $\{f_i : M \rightarrow N_i\}_{i = 1,\ldots,K}$. We equip every $n_i$-dimensional task manifold $N_i$ with a Riemannian metric $g_i$, a potential function $\Phi_i$, and dissipative forces $\mathcal F_{D,i}$.

We can associate to every robot position and velocity combination $v_p \in TM$ a smooth section $\gamma_{v_p,i} : [0,1] \longrightarrow f^*TN_i$ satisfying the single-task dynamics of the PBDS specified by $(f_i, g_i, \Phi_i, \mathcal F_{D,i})$. Extending the construction of (\ref{eq:appx_pullback_accelerations}) to an operator $G_i$ for each task, we collect the vectors $\gamma'_{{v_p}, i}(0) = ((p,v),(\dot\gamma^v_{{v_p}, i}(0),\dot\gamma^a_{{v_p}, i}(0)))$ into a globally well-defined and smooth operator
\begin{align}
S_i : TM &\longrightarrow T\big(f_i^*T N_i\big) \\
(p,v) &\mapsto \pi_\textrm{VB}(G_i(p,v)) = \pi_\textrm{VB}(\dot\gamma_{v_p,i}(0))\nonumber\\
&\hspace{73pt}= ((f_i^*df_i)_p(v), (0, \dot \gamma^a_{{v_p}, i}(0))),\nonumber
\end{align}
where $\pi_\textrm{VB}$ denotes the globally well-defined projection onto the vertical bundle. Recall that the vertical bundle of a bundle is the subbundle formed by the kernel of the differential of its standard projection, in this case $\textrm{ker}(\pi_{f_i^*T N_i})$. In particular, locally we have $\pi_\textrm{VB}(b^i \partial^v_i + a^i \partial^a_i) = a^i \partial^a_i$ for $T(f_i^*T N_i)$.

Next, we form a map relating robot accelerations on $TTM$ to their resulting task accelerations in $T(f^*TN_i)$:
\begin{equation}
\begin{aligned}
Z_i : TTM &\longrightarrow T(f_i^*T N_i) \\
\big( (p,v) , a \big) &\mapsto \pi_{\textrm{VB}}\Big( d(f_i^*df_i)_{(p,v)}(a) \Big).
\end{aligned}
\end{equation}
Since the pullback differential $f_i^*df_i$ is a smooth map between smooth manifolds $M$ and $f_i^*TN_i$, the differential of the pullback differential $d(f_i^*df_i)$ is globally well-defined. Thus $Z_i$ is globally well-defined and smooth.

We also assign a weighting pseudometric $w_i$ to each task tangent bundle $TN_i$. Given the higher-order task maps
\begin{equation}
    \begin{aligned}
    F_i : TM &\longrightarrow TN_i \\
    (p,v) &\mapsto (f_i(p), (df_i)_p(v)),
    \end{aligned}
\end{equation}
we can define globally well-defined pullback pseudometrics
\begin{equation}
\begin{aligned}
    F_i^*w_i : \; &T(f_i^*TN_i) \times T(f_i^*TN_i) \longrightarrow \mathbb{R} \\
    &((p,v),a_1), ((p,v),a_2) \mapsto (w_i)_{F_i(p,v)}(a_1,a_2).
\end{aligned}
\end{equation}

Now we define a function which will choose a robot acceleration for each robot position and velocity in $TM$ and will thus drive our PBDS dynamics:
\begin{equation}
\label{eq:optimization}
\begin{aligned}
    \zeta : TM &\longrightarrow \mathcal D \subset TTM\\
    (p,v) &\mapsto \underset{a \in \mathcal D_{(p,v)}}{\arg \min} \ \sum^K_{i=1} \frac{1}{2} \lVert Z_i(a) - S_i(p,v)   \rVert^2_{F_i^*w_i},
\end{aligned}
\end{equation}
where $\mathcal D$ is the globally well-defined affine distribution of $TTM$ such that the subspace $\mathcal D_{(p,v)} \subseteq T_{(p,v)}TM$ satisfies $a^v = v$ componentwise for each $a = ((p, v), (a^v, a^a)) \in \mathcal D_{(p,v)}$. Because all the involved quantities are globally defined, as soon as \eqref{eq:optimization} is well-defined (i.e., the minimization problem has a unique solution for every $(p,v) \in TM$) it is automatically globally defined.

To investigate conditions that guarantee existence and uniqueness of solutions to this minimization problem for every $(p,v) \in TM$, we derive solutions to (\ref{eq:optimization}) locally. In doing so, we also prove that \eqref{eq:optimization} is smooth. First, in order to unpack $Z_i(a)$, we note that the differential $d(f_i^*df_i)$ locally takes the form
\begin{equation}
\bm J(\bm f_i^* \bm{df}_i)(p,v) = \begin{bmatrix}
\bm{Jf}_i(p) & 0\\
\dot{\bm{Jf}}_i(p,v) & \bm{Jf}_i(p)\end{bmatrix},
\end{equation}
where for $(p,v)\in TM$ we define $\dot{Jf_i}$ locally as
$$
(\dot{Jf_i})_{kj}(p,v) = v^\ell \frac{\partial^2 f^j_i}{\partial x^\ell \partial x^k}(p).
$$
Now for $a \in \mathcal D_{(p,v)}$ we have
\begin{align*}
    Z_i&(a) = \pi_{\textrm{VB}}\left( d(f_i^*df_i)_{(p,v)}(a)\right)\\
    &= \pi_{\textrm{VB}}\big(((f_i^*df_i)_p(v), ((Jf_i)_{jk}(p)v^k, \\
    &\qquad\qquad(\dot{Jf}_i)_{jk}(p,v)v^k + (Jf_i)_{jk}(p)(a^a)^k))\big)\\
    &= ((f_i^*df_i)_p(v), (0,(\dot{Jf}_i)_{jk}(p,v)v^k + (Jf_i)_{jk}(p)(a^a)^k).
\end{align*}
Thus for $a \in \mathcal D_{(p,v)}$ we can locally compute
\begin{align}
    &\lVert Z_i(a) - S_i(p,v) \rVert^2_{F_i^*w_i} \\
    &= \lVert \bm{Jf}_i(p)\bm a^a + \dot{\bm{Jf}}_i(p,v)\bm v - \dot{\bm\gamma}^a_{v_p,i}(0)\rVert^2_{\bm w_i^a(F_i(p,v))}.\nonumber
\end{align}
using the notation $\lVert \bm x \rVert^2_{\bm B} = \bm x^\top \bm B \bm x$. Additionally for $\dot\gamma^a_{v_p,i}(0)$ we can locally compute
\begin{align*}
    (\dot\gamma^a_{v_p,i})^k&(0) = (f_i^*\mathcal F_i(\gamma_{v_p,i}(0))^\sharp)^k - v^\ell(t) \gamma^j_{v_p,i}(0) f_i^* (\Gamma_i)^k_{\ell j}(p)\\
    &= f^*_i \mathcal F_{D,i}(\gamma_{v_p,i}(0))^\sharp - f^*_i \textrm{grad}\,\Phi_i(p)\\
    &\quad- v^\ell(t) \gamma^j_{v_p,i}(0) f_i^* (\Gamma_i)^k_{\ell j}(p)\\
    &= g_i^{kj}(f_i(p))\bigg(\mathcal F_{D,i}^j((df_i)_p(v)) - \frac{\partial\Phi_i}{\partial y^j_i}(f_i(p))\bigg)\\
    &\quad -v^\ell(t)(Jf_i)_{jh}(p) v^h(t) (Jf_i)_{r\ell}(p)(\Gamma_i)_{rj}^k(f_i(p)).
\end{align*}
Thus, we can compute solutions to (\ref{eq:optimization}) locally by
\begin{align}
\bm\zeta(p,v) = \Bigg( \sum_{i=1}^K &\bm{Jf}_i(p)^\top\bm w_i^a(F_i(p,v))\bm{Jf}_i(p)\Bigg)^\dagger\\
&\left(\sum_{i=1}^K \bm{Jf}_i(p)^\top \bm w_i^a(F_i(p,v))\bm{\mathcal{A}}_i(p,v)\right),\nonumber
\end{align}
where $\bm{\mathcal{A}}_i(p,v) = \dot{\bm{Jf}}_i(p,v)\bm v(t) - \dot{\bm\gamma}^a_{v_p,i}(0)$. We can now see that that for $\zeta$ to be smooth and always have a unique solution, the sum within the pseudoinverse must be full rank. To achieve this, we make some key assumptions about the joint design of the task maps $f_i$ and weighting pseudometrics $w_i$ which will also be important for the stability results. First for convenience, we define a function giving the indices of tasks having nonzero weights:
\begin{equation}
\begin{aligned}
\label{eq:indices}
\mathcal I : TM &\longrightarrow \mathcal P(\{1,\ldots,K\})\\
(p,v) &\mapsto \{i \in \mathbb \{1,\ldots,K\} \ | \ w_i((df_i)_p(v)) \neq 0\}.
\end{aligned}
\end{equation}
This allows us to build a family of product task maps $ f_{(p,v)} \triangleq \prod_{i\in\mathcal I(p,v)} f_i$
giving for each $(p,v) \in TM$ the product map of the task maps associated to nonzero weights.

Now we make the following assumptions:
\begin{itemize}
    \item[$(A1)$] The pseudometrics $w_i$ are at every point in $TN_i$ either positive-definite or zero.
    \item[$(A2)$] The differential $(df_{(p,v)})_p$ is always of rank $m$.
\end{itemize}
Note that $(A2)$ implies that $m \le \sum^K_{i=1}n_i$. Now for all $(p,v) \in TM$ we can compute
\begingroup
\allowdisplaybreaks
\begin{align*}
    &\qquad\sum_{i=1}^K \bm{Jf}_i(p)^\top\bm w_i^a(F_i(p,v))\bm{Jf}_i(p)\\
    &= \sum_{i\in\mathcal I(p,v)} \bm{Jf}_i(p)^\top\bm w_i^a(F_i(p,v))\bm{Jf}_i(p)\\
    &= \bm{Jf}_{(p,v)}^\top(p) \bm w^a_{(p,v)}(F_{(p,v)}(p,v)) \bm{Jf}_{(p,v)}(p),
\end{align*}
\endgroup
where $ F_{(p,v)} \triangleq \prod_{i\in\mathcal I(p,v)} F_i$ and $\bm w^a_{(p,v)}(F_{(p,v)}(p,v)) \triangleq \textrm{blockdiag}( \{\bm w^a_i(F_i(p,v))\}_{i\in\mathcal I(p,v)})$. Given $(A1)$ and $(A2)$, the final local matrix is positive definite, so the original sum is indeed full rank, giving us unique solutions and smoothness of $\zeta$. Thankfully, these assumptions are easily satisfied in practice as explained in Sec \ref{sec:pbds}.

However, it is useful to show that in general, with probability one the local matrix $\bm{Jf}_{(p,v)}(p)$ has full rank at every point $p \in M$ if the size of $\mathcal I(p,v)$ is large enough (in particular, if $\textrm{card}(\mathcal I(p,v)) \ge 2m$ for every $(p,v) \in TM$). This may be inferred as a consequence of the Whitney immersion theorem. However, by leveraging transversality theory we can directly prove that, i.e., for a sufficiently high number of tasks, the required full-rank properties are satisfied pointwise with probability one.

To provide a more precise statement, we proceed in steps.
Our objective is achieved by selecting at least $2 m$ tasks, if we can shown that the complementary of the set
\begin{align*}
    \mathcal{S} \triangleq \Big\{ &(f_1,\dots,f_{2 m}) \in C^{\infty}(M;N_1) \times \\
    &\dots \times C^{\infty}(M;N_{2 m}) : \ \textrm{rank}(f^{2 m}) < m \Big\}
\end{align*}
is dense with respect to the Whitney topology (see, e.g., \cite{Hirsch1976}). We prove that the complementary of the set $\mathcal{S}$ is dense in $C^{\infty}(M;N_1) \times \dots \times C^{\infty}(M;N_{2 m})$ with respect to the Whitney topology by the jet transversality theorem (see, e.g., \cite{Hirsch1976}). For this, we denote
$$
\theta(f_1,\dots,f_{2 m}) \triangleq \Big( \tilde \theta , \theta^{f_1}_0 , \dots , \theta^{f_{2 m}}_0 , \theta^{f_1}_1 , \dots , \theta^{f_{2 m}}_1 \Big)
$$
as the jet of order one of $f_1,\dots,f_{2 m}$. Let us define the following smooth mapping in local jet coordinates
$$
\rho(\theta) \triangleq \Big( \theta^{f_1}_1 , \dots , \theta^{f_{2 m}}_1 \Big) .
$$
It is clear that $\rho$ is a local submersion. It is also clear that locally we have
$$
\mathcal{S} = \underset{1 \le i \le m-1}{\bigcup} \rho^{-1}(L(2m,m;i)),
$$
where $L(2m,m;i)$ denotes the Stiefel manifold of linear mappings from $\mathbb{R}^m$ to $\mathbb{R}^{2m}$ of rank $1 \le i \le m$. Therefore, in terms of codimension, we have
\begin{align*}
    \textrm{codim}(\mathcal{S}) &\ge \textrm{codim}\big( \rho^{-1}(L(2m,m;m-1)) \big) \\
    &= \textrm{codim}\big( L(2m,m;m-1) \big) = m+1 .
\end{align*}
Then, the conclusion comes from a direct application of the jet transversality theorem (see, e.g., \cite{Hirsch1976}).

Now using the above constructions, we finally give the definition of a multi-task PBDS:

\begin{definition}[Multi-Task Pullback Bundle Dynamical System] Let $\{f_i : M \longrightarrow N_i\}_{i=1,\ldots,K}$ be smooth task maps for a Riemannian task manifolds $(N,g)$, with corresponding smooth potential functions $\Phi_i : N_i \longrightarrow \mathbb R_+$, dissipative forces $\mathcal F_{D,i} : TN_i \longrightarrow T^*N_i$, and weighting pseudometrics $w_i$ on $TN_i$. Then the set $\{(f_i, g_i, \Phi_i, \mathcal F_{D,i}, w_i)\}_{i=1,\ldots,K}$ forms a multi-task PBDS with curves $\sigma : [0,\infty) \longrightarrow M$ satisfying
\begin{equation}
\label{eq:appx_multitask_pbds}
\begin{cases}
\ddot{\sigma}(t) = \zeta(\dot{\sigma}(t)) = \\
\qquad \underset{a \in \mathcal D_{\dot\sigma (t)}}{\arg \min} \ \sum^K_{i=1} \frac{1}{2} \lVert Z_i(a) - S_i(\dot\sigma(t))   \rVert^2_{F_i^*w_i} \\
\dot\sigma(0) = (p_0, v_0).
\end{cases}\hspace{-10pt}
\end{equation}
\end{definition}
Assuming $(A1)$ and $(A2)$ to ensure smoothness of $\zeta$ and the existence of unique solutions to (\ref{eq:optimization}), we see that we have existence, uniqueness, and smoothness of solution curves $\sigma$ for a multi-task PBDS.


\subsection{Proof of Multi-Task PBDS Stability}
\label{appxsubsec:pbds_stability}
This section proves the stability results summarized in the main body of the paper.
First, we make another key assumption related to the design of weighting pseudometrics. Using the convenience function in (\ref{eq:indices}), for dissipative forces we define $\mathcal F_{D,(p,v)} \triangleq \prod_{i\in\mathcal I(p,v)} \mathcal F_{D,i}$.
We now assume that for every $(p,v) \in TM$ we have the following:
\begin{itemize}
    \item[$(A3)$] The dissipative force $\mathcal F_{D,(p,v)}$ is strictly dissipative, meaning  $\langle\mathcal F_{D,(p,v)}\big((df_{(p,v)})_p(v)\big),(df_{(p,v)})_p(v)\rangle < 0$ for $v \neq 0$,
\end{itemize}
Our Lyapunov stability results are as follows:
\begin{proposition}[Lyapunov Function for Multi-Task PBDS]
\label{prop:lyapunov_function}
Let  $\{(f_i, g_i, \Phi_i, \mathcal F_{D,i}, w_i)\}_{i=1,\ldots,K}$ be a multi-task PBDS. Then given the function $V : TM \longrightarrow [0,\infty)$ defined by
\begin{equation}
V(p,v) = \sum^K_{i=1}\frac{1}{2} \big\lVert (f_i^*df_i)_p(v) \big\rVert^2_{(f_i^* g_i)_p} + \Phi_i \circ f_i(p),
\end{equation}
we have that
$$
\frac{dV}{dt}(\dot{\sigma}(t)) = \sum^K_{i=1} \langle\mathcal F_{D,i}((df_i)_{\sigma(t)}(\dot\sigma(t))),(df_i)_{\sigma(t)}(\dot\sigma(t))\rangle < 0
$$
for $\dot\sigma(t) \neq 0$ along solution curves $\sigma$ to the multi-task PBDS.
\end{proposition}
\begin{proof}
Taking one element of the sum in $\frac{d}{dt}V(\sigma'(t))$ from the first term we have
\begin{align}
    &\frac{d}{dt} \Big( \frac{1}{2} \big\lVert (f_i^*df_i)_{\sigma(t)}(\dot\sigma(t)) \big\rVert^2_{(f_i^* g_i)_{\sigma(t)}}\Big)\nonumber \\
    &=\frac{d}{dt}\Big(\frac{1}{2}(f_i^* g_i)_{\sigma(t)} \big( (f_i^*df_i)_{\sigma(t)}({\dot\sigma(t)}) , (f_i^*df_i)_{\sigma(t)}(\dot\sigma(t)) \big)\Big)\nonumber\\
    \label{eq:Vdot1}
    &=(f_i^*g_i)_{\sigma(t)}(f_i^*D_{\sigma,i}\gamma_{\sigma,i}(0), \gamma_{\sigma,i}(0))\\
    &=(f_i^*g_i)_{\sigma(t)}(f_i^*\mathcal F_{D,i}(\gamma_{\sigma,i}(0))^\sharp,\gamma_{\sigma,i}(0))\nonumber\\
    &\quad - (f_i^*g_i)_{\sigma(t)}(f^*\textrm{grad}\,\Phi_i(\sigma(t)),\gamma_{\sigma,i}(0))\nonumber\\
    &=\langle f_i^*\mathcal F_{D,i}(\gamma_{\sigma,i}(0)), \gamma_{\sigma,i}(0)\rangle\nonumber\\
    &\quad - \langle f_i^*\textrm{grad}\Phi_i(\sigma(t))^\flat,\gamma_{\sigma,i}(0)\rangle\nonumber\\
    &=\langle\mathcal F_{D,i}((df_i)_{\sigma(t)}(\dot\sigma(t))),(df_i)_{\sigma(t)}(\dot\sigma(t))\rangle\nonumber\\
    &\quad - B_{\Phi_i}(\sigma(t))\big( (f_i^* df_i)_{\sigma(t)}(\dot\sigma(t)) \big)\nonumber\\
    \label{eq:Vdot2}
    &= \langle\mathcal F_{D,i}((df_i)_{\sigma(t)}(\dot\sigma(t))),(df_i)_{\sigma(t)}(\dot\sigma(t))\rangle\\
    &\quad- d(\Phi_i\circ f_i)(\dot\sigma(t)),\nonumber
\end{align}
where for (\ref{eq:Vdot1}) we use the compatibility of the pullback metric and the pullback connection, and for (\ref{eq:Vdot2}) we use (\ref{eq:B_Phi}). Likewise, for the second term we simply have
$$
\frac{d}{dt}\Big(\Phi_i\circ f_i(\sigma(t))\Big) = d(\Phi_i\circ f_i)(\dot\sigma(t)).
$$
Thus, we have
$$
\frac{d}{dt}\big(V(\dot{\sigma}(t))\big) = \sum^K_{i=1} \langle\mathcal F_{D,i}((df_i)_{\sigma(t)}(\dot\sigma(t))),(df_i)_{\sigma(t)}(\dot\sigma(t))\rangle.
$$
Then from assumption $(A2)$ and the fact that all forces $\mathcal F_{D,i}$ are dissipative, we have that $\dot V(\dot{\sigma}(t)) < 0$ for $\dot{\sigma}(t) \neq 0$.
\end{proof}

We now provide global stability results:
\begin{theorem}[Global Stability of Multi-Task PBDS]
\label{appxthe:pbds_stability}
Let  $\{(f_i, g_i, \Phi_i, \mathcal F_{D,i}, w_i)\}_{i=1,\ldots,K}$ be a multi-task PBDS where $\sum_{i=1}^K \Phi_i \circ f_i$ is a proper map.
Then given the Lyapunov function $V$ of Prop. \ref{prop:lyapunov_function}, the multi-task PBDS satisfies these:
\begin{itemize}
    \item The sublevel sets of $V$ are compact and positively invariant for $\zeta : TM \longrightarrow TTM$, so that every solution curve $\sigma$ is defined in the whole interval $[0,+\infty)$.
    \item For every $\beta > 0$, every solution curve $\sigma$ in $M$ starting at $\dot{\sigma}(0) \in V^{-1}([0,\beta])$ converges in $V^{-1}([0,\beta])$ as $t\rightarrow +\infty$ to an equilibrium set $\{(p,0)\in V^{-1}([0,\beta])$ : $f^*_i \textnormal{grad}\,\Phi_i(p) = 0$ if $w_i(f_i(p),0) \neq 0\}$.
\end{itemize}
\end{theorem}

\begin{proof}
First, we aim to show that the sublevel sets of $V$ are compact. Following the argument in the proof of \cite[Theorem 6.47]{BulloLewis2004}, since we know $\sum_{i=1}^K \Phi_i \circ f_i$ is proper, it is sufficient to show that
\begin{align*}
f^*g : TM \times TM &\longrightarrow \mathbb R\\
((p,v_1), (p,v_2))  &\mapsto \sum^K_{i=1} (f_i^* g_i)_p \big( (f_i^*df_i)_p(v_1) , (f_i^*df_i)_p(v_2) \big).
\end{align*}
defines a metric.
Symmetry and bilinearity follow easily from the properties of the task pullback metrics $f^*g_i$. For positive-definiteness, we first take assumption $(A2)$, which also gives that the differential $df_p$ is always of rank $m$, for the product map $f\triangleq \prod_{i=1}^K f_i$. Now note that if $f^*g_p(v,v) = 0$, then we must have $(f_i^* g_i)_p \big( (f_i^*df_i)_p(v) , (f_i^*df_i)_p(v) \big) = 0$ for each task. With $df_p$ of full rank, this implies that $v = 0$. Thus it is clear that $f^*g$ is positive-definite and is a metric, which in turn implies that
$V$ is proper. In addition, since we have $\dot V(\dot{\sigma}(t)) < 0$ from Proposition \ref{prop:lyapunov_function}, the level sets of $V$ are positively invariant.

Now, let $\beta > 0$ and define the set
\begin{align*}
    A &= \Big\{ (p,v) \in V^{-1}([0,\beta]) : \\
    &\hspace{30pt}\langle\mathcal F_{D,(p,v)}\big((df_{(p,v)})_p(v)\big),(df_{(p,v)})_p(v)\rangle= 0 \Big\} \\
    & = \Big\{ (p,0) \in V^{-1}([0,\beta]) \Big\}.
\end{align*}
where the last equality follows from $(A3)$. First, we wish to show that the set
\begin{align*}
    B = \Big\{ (p,0) \in &V^{-1}([0,\beta]) : \\
    &f_i^*\textrm{grad} \Phi_i(p) = 0 , \ w_i(f_i(p),0) \neq 0 \Big\}
\end{align*}
is the largest positively invariant set in $A$ for the multi-task PBDS. By its definition, we have $B \subseteq A$. Now suppose by contradiction that we have a point $(p,0)$ in the largest positively invariant set of $A$ such that $(p,0) \notin B$. Then $f^*_i\textrm{grad}\,\Phi_i(p) \neq 0$ for some $i$ such that $w_i(f_i(p),0) \neq 0$. Let the multi-task PBDS system evolve starting from $(p,0)$, and let $\sigma_{(p,0)}$ be its solution. If $\dot{\sigma}_{(p,0)}(t) = 0$ for all $t \in [0,+\infty)$, then $\dot{\sigma}_{(p,0)}(0) = 0$, and $\sigma_{(p,0)}(t) = p$ does not satisfy the system. Thus, there must be $\bar t > 0$ such that $\dot{\sigma}_{(p,0)}(\bar t) \neq 0$, in contradiction with $\sigma_{(p,0)}(t) \in A$ for every $t > 0$.

Now using the LaSalle Invariance Principle, we can conclude that for solution curves $\sigma$ to the multi-task PBDS having initial condition $\dot{\sigma}(0) \in V^{-1}([0,\beta])$, we have $\underset{t \rightarrow +\infty}{\lim} \ \textrm{dist}(\dot{\sigma}(t),B) = 0$, and the conclusion follows.
\end{proof}

There are a couple key things to highlight about this result. First, it is important to note that it relies on $\sum_{i=1}^K \Phi_i \circ f_i$  being a proper map. If it is not proper but its restrictions to some sublevel set of the Lyapunov function is proper, then convergence is only guaranteed within that sublevel set.

Second, given the assumptions, convergence to an equilibrium state is guaranteed. However, in some cases there can be zero-velocity equilibrium states at local maxima or saddle points of the attractor potential function. Fortunately in many cases these are sets of zero measure that can be escaped by injecting small disturbances into the system when such a situation is detected.


\section{Mathematical Details for PBDS Policies on Tangent Bundle Tasks}
\label{appxsec:tangent_bundle_pbds}
Here we provide further details for the extension of the PBDS framework to tangent bundle tasks, as used in Sec.~\mbox{\ref{subsec:task_priorities_vel_dep_metrics}}. This framework is mainly instructive to show the challenges both in performing such an extension in a way that is geometrically well-defined and in finding practical application for using velocity-dependent Riemannian metrics to drive task behavior. Indeed, in this work we were not able to find useful robotic tasks whose behavior could not be replicated using the basic multi-task PBDS framework.

Given a task map $f : M \longrightarrow N$, consider the higher-order task map $F : TM \longrightarrow TN : (p,v) \mapsto (f(p), df_p(v))$. It would be convenient if we could follow the same steps we used to extract a robot acceleration policy from task maps mapping from $M$. However, task maps mapping from $TM$ introduce two major new challenges.

First, the previous procedure applied in this case would produce curves on $TM$ rather than on $M$ as we desire. Unfortunately, not every curve in $TM$ correctly represents a curve in $M$. In particular, given a curve $\mu : I \longrightarrow TM$, the velocity curve $\dot{\mu}$ takes points in $TTM$, which have the form $((p,v),(b,a))$, where $(p,v)$ is the base point from $TM$ and $(b,a)$ is the vector portion. If $\mu$ is to correctly represent a curve $\sigma : I \longrightarrow M$, we must have that $\dot{\mu}(t) = ((\sigma(t),\dot\sigma(t)), (\dot\sigma(t),\ddot\sigma(t)))$ for every $t \in I$. This clearly imposes a velocity constraint on admissible curves in $TM$ that we can consider: namely, points $((p,v),(b,a))$ in the velocity of admissible curves curves in $TM$ must satisfy $v = b$ (in local coordinates). The way we choose to handle this is through a subbundle constraint, as we will see.

Second, since velocities in this case are on $TTM$, an acceleration operator applied to a curve $\mu$ in $TM$ gives twice as many acceleration vectors as an acceleration operator applied to a curve in $M$. In particular, if $\mu$ should represent a curve $\sigma$ in $M$, the acceleration operator gives vectors that should represent the second and third time derivatives of $\sigma$, i.e. the acceleration on the jerk. Thus to meet our goal of an acceleration policy, it is desirable to project away the jerk components. The GDS framework in RMPflow accomplished this by applying a projection onto the horizontal bundle~\cite{RatliffIssacEtAl2018}. However, as discussed previously, there is in general no choice of horizontal bundle that is disjoint with the vertical bundle in all charts, so this projection is not globally well-defined.

Instead, we choose to handle this by applying a permutation to swap the position of the jerk and acceleration vectors, an operation which is globally well-defined for parallelizable manifolds that are embedded in the Euclidean space, such as Lie groups. Then the global vertical bundle projection will correctly remove the jerk components, leaving us with the acceleration components that we can use to build an acceleration policy. Note that an alternative is to apply the vertical bundle projection without a permutation and recover a jerk policy, but leave that alternative open to further study.

With this motivation and roadmap, we begin by considering enforcement of velocity constraints.
 
\subsection{Enforcing Velocity Constraints}
We will enforce the velocity constraints discussed above using subbundle constraints.
For this, we first consider the augmented task map
\begin{align*}
F \times \pi_M : TM &\longrightarrow TN \times M\\
(p,v) &\mapsto (F(p,v),\pi_M(p,v)) = ((p, df_p(v)), p).
\end{align*}
This would suffice if we aimed to recover a jerk policy. However, now consider the permutation map $\sigma_{TN} : TN \longrightarrow \mathbb R^{2d} : (p,v) \mapsto (d\bar \varphi_p(v), \bar \varphi(p))$, where $N$ has an embedding $\bar \varphi : N \longrightarrow \mathbb R^d$. Note that this map is globally well-defined for parallelizable manifolds that are embedded in Euclidean space, such as all Lie groups (though unfortunately, a notable non-parallelizable manifold is $\mathbb S^2$). Using this and the embedded task map $\bar f : M \longrightarrow \mathbb R^d$, we can form the composition
\begin{align*}
    \hat F \triangleq \sigma_{TN} \circ F : TM &\longrightarrow \mathbb R^{2d}\\
    (p,v) &\mapsto (d\bar f_p(v),\bar f(p))
\end{align*}
We now construct the full augmented permuted task map which we will use:
\begin{align*}
    \breve F \triangleq \hat F \times \pi_M : TM &\longrightarrow \mathbb R^{2d} \times M \\
    (p,v) &\mapsto (d\bar f_p(v), \bar f(p),p).
\end{align*}
Now given the modified task tangent bundle $T(\mathbb R^{2d} \times M)$, we can construct a pullback bundle $\breve F^*T(\mathbb R^{2d} \times M)$. Note that this is a vector bundle which is isomorphic to the pullback direct sum vector bundle $\hat F^*(T\mathbb R^{2d}) \oplus \pi_M^*TM$.

Thus using straightforward identifications, we can define the map
\begin{align*}
\hat F^*\mathcal L : TTM &\longrightarrow \hat F^*T\mathbb R^{2d} \oplus \pi_M^*TM\\
((p,v),(b,a)) & \mapsto ((p,v), \\
&\qquad d\hat F_{(p,v)}(b,a) + ((d\pi_M)_{(p,v)}(b,a) - v)).
\end{align*}
We can also define the two vector subbundles of the pullback bundle $\breve F^*T(\mathbb R^{2d} \times M)$:
\begin{align*}
    &\breve F^*G_\textrm{VB} = \Big\{((p,v), c^1_i \partial x^i_{\mathbb R^{2d}} + c^2_i \partial v^i_{\mathbb R^{2d}} + c_i^3 \partial x^i_M) \\
    &\qquad\quad\in \hat F^*T\mathbb R^{2d} \oplus \pi_M^*TM\\
    &\qquad\quad| \ c_i^1 = 0 \textrm{ for } i = 1,\ldots, d, c_i^3 = 0 \textrm{ for } i = 1,\ldots, m \Big\},\\
    &\breve F^*G_{\mathbb R^{2d}} = \Big\{((p,v), c^1_i \partial x^i_{\mathbb R^{2d}} + c^2_i \partial v^i_{\mathbb R^{2d}} + c_i^3 \partial x^i_M) \\
    &\qquad\quad\in \hat F^*T\mathbb R^{2d} \oplus \pi_M^*TM \ | \  c_i^3 = 0 \textrm{ for } i = 1,\ldots, m \Big\},
\end{align*}
which are globally well-defined and smooth, in the first case due to its vertical bundle structure and in the second case due to standard projection arguments exploiting the direct sum structure of the pullback bundle described earlier. Note that we have $\breve F^*G_\textrm{VB} \subseteq \breve F^*G_{\mathbb R^{2d}}$ as vector bundles. We also use these to define a further subbundle of the pullback bundle $\breve F^*G = \breve F^*G_\textrm{VB} \oplus \breve F^*G_{\mathbb R^{2d}}^\bot$, where $\breve F^*G_{\mathbb R^{2d}}^\bot$ is the orthogonal complement of $\breve F^*G_{\mathbb R^{2d}}$ in the pullback bundle $\breve F^*T(\mathbb R^{2d} \times M)$. This orthogonal complement can be defined by
\begin{align*}
\breve F^*G_{\mathbb R^{2d}}^\bot = \{&(p,v) \in \breve F^*T(\mathbb R^{2d} \times M)\\
&: \breve F^*g_p(v,w)=0, \forall w \in \breve F^*G_{\mathbb R^{2d}}  \},
\end{align*}
given some pullback metric $\breve F^*g$ for $F^*T(\mathbb R^{2d} \times M)$.

Now suppose we have a robot acceleration $\sigma'' = ((p,v),(b,a)) \in TTM$, with corresponding point $\hat F^*\mathcal L(\sigma'') \in \breve F^*T(\mathbb R^{2d} \times M)$ in the pullback bundle. If we also have $\hat F^*\mathcal L(\sigma'') \in \breve F^*G_{\mathbb R^{2d}}$, this implies that $(d\pi_M)_{(p,v)}(b,a) - v = 0$, or componentwise $b-v=0$. This is exactly the velocity constraint we aim to impose. Thus, we will enforce this constraint by ensuring our dynamics evolve on the subbundle $\breve F^*G_{\mathbb R^{2d}}$.

In order to do so, we need to develop a connection on the pullback bundle $\breve F^*T(\mathbb R^{2d} \times M)$ which is guaranteed to provide curves that stay in the subbundle $\breve F^*G_{\mathbb R^{2d}}$ if they start in the subbundle. First we define a projection $P_{\breve F^*G_{\mathbb R^{2d}}^\bot}$ from $\breve F^*T(\mathbb R^{2d} \times M)$ onto $\breve F^*G_{\mathbb R^{2d}}^\bot$, which is easily seen to be globally well-defined. Now consider the tensorization of the $P_{\breve F^*G_{\mathbb R^{2d}}^\bot}$ operator:
\begin{align*}
    P_{\breve F^*G_{\mathbb R^{2d}}^\bot}^T : TM &\longrightarrow \breve F^*T(\mathbb R^{2d} \times M) \otimes \breve F^*T^*(\mathbb R^{2d} \times M)\\
    (p,v) &\mapsto \breve F^*E^*_i(v_p)\left( P_{\breve F^*G_{\mathbb R^{2d}}^\bot}(\breve F^*E_j(v_p))\right)\\
    &\hspace{16pt} \breve F^*E_i(v_p) \otimes \breve F^*E^*_j(p),
\end{align*}
where $(\breve F^*E^*_i)$ is the dual local frame of $\breve F^*T^*(\mathbb R^{2d} \times M)$ associated to the local frame $(\breve F^*E_i)$ of $\breve F^*T(\mathbb R^{2d} \times M)$.

Now let $\nabla_{\mathbb R^{2d}}$ and $\nabla_M$ be connections over $\mathbb R^{2d}$ and $M$, respectively (e.g., $\nabla_{\mathbb R^{2d}}$ can be the pushforward of a connection on $TN$ through $\sigma_{TN}$).
Then we can combine $\nabla_{\mathbb R^{2d}}$ and $\nabla_M$ to form a connection $\breve \nabla$ over $\mathbb R^{2d} \times M$. From here, we can build a pullback connection $\breve F^*\breve\nabla$ over $TM$ using the usual definition, having associated Christoffel symbols $\breve F^*\breve\Gamma^k_{ij}$.

Then for every vector field $X$ over $TM$ and section $V$ of $\breve F^*T(\mathbb R^{2d} \times M)$ we can build a well-defined operator 
$$
\breve F^*\breve\nabla P_{\breve F^*G_{\mathbb R^{2d}}^\bot}^T(\cdot,\cdot) : TM \rightarrow \breve F^*T(\mathbb R^{2d} \times M) \otimes \breve F^*T^*(\mathbb R^{2d} \times M)
$$
such that we may consider the identification
$$
\breve F^*\breve\nabla P_{\breve F^*G_{\mathbb R^{2d}}^\bot}^T(\cdot,V) : TM \rightarrow \breve F^*T(\mathbb R^{2d} \times M).
$$
From this we can compute the coordinate representation
\begin{align*}
    \breve F^*\breve\nabla P_{\breve F^*G_{\mathbb R^{2d}}^\bot}^T&(\cdot,V) = (\mathds 1_{\{h > 2d, k \leq 2d\}}(h,k)\\
    &- \mathds 1_{\{h\leq 2d, k > 2d\}}(h,k)) \breve F^* \breve\Gamma^k_{jh} X^jV^h\mathfrak n_k,
\end{align*}
where $\mathds 1$ is an indicator function and $(\mathfrak n_i)$ is an orthonormal basis for $\breve F^*T^*(\mathbb R^{2d} \times M)$. This allows us to define a new connection over the pullback bundle $\breve F^*T^*(\mathbb R^{2d} \times M)$ as
$$
\overset{\breve F^*G_{\mathbb R^{2d}}}{\breve F^*\breve\nabla_X} V \triangleq \breve F^*\breve\nabla_X V +  \breve F^*\nabla P_{\breve F^*G_{\mathbb R^{2d}}^\bot}^T(\cdot,V),
$$
whose Christoffel symbols take the form
\begin{equation}
\label{eq:constr_christoffel}
\begin{aligned}
\overset{\breve F^*G_{\mathbb R^{2d}}}{\breve F^*\breve\Gamma^k_{jh}} = (1 + &\mathds 1_{\{h > 2d, k \leq 2d\}}(h,k)\\
&- \mathds 1_{\{h \leq 2d, k > 2d\}}(h,k))\breve F^*\breve\Gamma^k_{jh}.
\end{aligned}
\end{equation}
In particular, notice when that when $V(\breve F(p,v)) \in \breve F^*G_{\mathbb R^{2d}}$ for every $(p,v) \in TM$, we have
\begin{align*}
\breve F^*\breve\nabla &P_{\breve F^*G_{\mathbb R^{2d}}^\bot}^T(\cdot,V)(\breve F(p,v))\\
&= -\sum^{2d+m}_{k=2d+1}\bigg(\sum^{2m}_{j=1} \sum^{2d}_{h=1}\breve F^*\breve\Gamma^k_{jh} X^j V^h\bigg)\mathfrak n_k \in \breve F^*G_{\mathbb R^{2d}}^\top.
\end{align*}
In short, this new connection $\overset{\breve F^*G_{\mathbb R^{2d}}}{\breve F^*\breve\nabla}$ effectively removes the connection terms which may cause a curve to leave the subbundle $\breve F^*G_{\mathbb R^{2d}}$ where our desired curves should stay. For brevity, we omit the full proof of this result which generalizes \cite[Prop. 4.85]{BulloLewis2004}.

\subsection{Multi-Task Acceleration Policy Optimization}
Now using this connection, we can build local pullback bundle dynamical systems having curves that satisfy our desired velocity constraints:
\begin{definition}[Local PBDS on Tangent Bundle Task] Let $f : M \longrightarrow N$ be a smooth task map, where $N$ is embedded in $\mathbb R^d$, and associate to $f$ the augmented, permuted task map $\breve F : TM \longrightarrow \mathbb R^{2d} \times M$. Also let $\breve g$ be a Riemannian metric on $\mathbb R^{2d} \times M$, let $\breve \Phi : \mathbb R^{2d} \times M \longrightarrow \mathbb R_+$ be a smooth potential function, and let $\breve{\mathcal F}_D : T(\mathbb R^{2d} \times M) \longrightarrow T^*(\mathbb R^{2d} \times M)$ be dissipative forces. Then for each $(p,v)\in M$, we can choose a curve $\alpha_{(p,v)} : (-\varepsilon, \varepsilon) \longrightarrow TM$ resulting in $\gamma_{\alpha_{(p,v)}} : (-\varepsilon, \varepsilon) \longrightarrow \breve F^*T(\mathbb R^d\times M)$ for some $\varepsilon > 0$ such that $(f,g,\Phi,\mathcal F_D, \alpha_{(p,v)})$ forms a local Pullback Bundle Dynamical System (PBDS) satisfying
\begin{equation*}
\textrm{PBDS}_{\alpha_{(p,v)}} \hspace{-5pt}\ \begin{cases}
    P_{\breve F^*G}\left(\overset{\breve F^*G_{\mathbb R^{2d}}}{\breve F^*D_{\alpha_{(p,v)}}}\left(P_{\breve F^*G}(\gamma_{\alpha_{(p,v)}}(s)\right)\right)\\
    \quad = P_{\breve F^*G}\left(\breve F^*\breve{\mathcal F}(\gamma_{\alpha_{(p,v)}}(s))^\sharp\right)\\
    \gamma_{\alpha_{(p,v)}}(0) = \breve F^*d\breve F_{\alpha_{(p,v)}(0)}(\alpha_{(p,v)}'(0)),\\
    \alpha_{(p,v)}(0) = (p,v),
\end{cases}
\end{equation*}
\end{definition}
\noindent where we denote total pullback acceleration $\breve F^*\breve{\mathcal F}(\cdot)^\sharp$, analogous to Appx. \ref{appxsec:single_pbds}. As before, given $(p,v)\in TM$, we have that $\gamma'_{\alpha_{(p,v)}}(0)$ from a corresponding local PBDS does not depend on the choice of curve $\alpha_{(p,v)}$, so we will again use the shorthand $\gamma'_{v_p}(0) \triangleq \gamma'_{\alpha_{(p,v)}}(0)$.

Now in order to extract a single robot acceleration policy from a set of tasks, we trace the constructions in Appx. \ref{appxsubsec:multi_pbds} with some modifications to form the required geometric optimization problem. For $K$ tasks, consider task maps $\{f_i : M \rightarrow N_i\}_{i = 1,\ldots,K}$, where each $N_i$ is embedded in $\mathbb R^{d_i}$, giving corresponding embedded task map $\bar f_i$. We equip each associated embedding space $\mathbb R^{2d_i}$ for $TN_i$ with a Riemannian metric $g_{\mathbb R^{2d_i}}$ and a potential function $\Phi_{\mathbb R^{2d_i}}$ depending only on $v$ for $(p,v)\in T\mathbb R^{2d_i}$. We also equip each $T\mathbb R^{2d_i}$ with a weighting pseudometric $w_{T\mathbb R^{2d_i}}$ and dissipative forces $\mathcal F_{D\mathbb, R^{2d_i}}$, both dependent only on the base points of $T\mathbb R^{2d_i}$, as will be important for defining our geometric optimization problem. We also set the first $d_i$ components of $\mathcal F_{D\mathbb, R^{2d_i}}$ to be zero. Next, we define similar objects $g_M$, $\Phi_M$ on $M$ and $w_M$, $\mathcal F_{D,M}$ on $TM$, although their design is inconsequential as their contributions are ultimately projected away.

From these we can form on each augmented, permuted task space $\mathbb R^{2d_i} \times M$ the product metric $\breve g_i = g_{\mathbb R^{2d_i}} \oplus g_M$ and potential forces $\breve \Phi_i$ such that $\breve \Phi_i(v,q,p) = \Phi_{\mathbb R^{2d_i}}(v,q) + \Phi_M(p)$. Likewise on each $T(\mathbb R^{2d_i} \times M)$ we form the weight product pseudometric $\breve w_i = w_{T\mathbb R^{2d_i}} \oplus w_{TM}$ and dissipative forces $\breve{\mathcal F}_{D,i} = \mathcal F_{D\mathbb, R^{2d_i}} \oplus \mathcal F_{D\mathbb, M}$.

Then we define two distributions of $T(\breve F_i^*T(\mathbb R^{2d_i}\times M)) \cong T(\hat F^*T\mathbb R^{2d_i}) \oplus T(\pi_M^*TM)$ analogous to the subbundles $\breve F^*G_\textrm{VB}$ and
$\breve F^*G_{\mathbb R^{2d}}$:
\begin{align*}
    &\mathcal D_{\textrm{VB},i} = \Big\{\big(((p,v),a), c^1_j \partial x^j_M + c^2_j \partial v^j_{M,1} + c^3_j \partial a^j_{\mathbb R^{2d_i}}\\ 
    &\quad+ c^4_j \partial v^j_{\mathbb R^{2d_i}} + c_j^5 \partial v^j_{M,2}\big) \in T(\hat F_i^*T\mathbb R^{2d_i}) \oplus T(\pi_M^*TM)\\
    &\quad| \ c_j^1 = c_j^2 = 0 \textrm{ for } j = 1,\ldots, m, c_j^3 = 0 \textrm{ for } j = 1,\ldots, d \Big\},
\end{align*}
\begin{align*}
    &\mathcal D_{\mathbb R^{2d_i}} = \Big\{\big(((p,v),a), c^1_j \partial x^j_M + c^2_j \partial v^j_{M,1} + c^3_j \partial a^j_{\mathbb R^{2d_i}}\\ 
    &\quad+ c^4_j \partial v^j_{\mathbb R^{2d_i}} + c_j^5 \partial v^j_{M,2}\big) \in T(\hat F_i^*T\mathbb R^{2d_i}) \oplus T(\pi_M^*TM)\\
    &\quad | \ c_j^5 = 0 \textrm{ for } j = 1,\ldots, m \Big\}
\end{align*}
and define projections $P_{\mathcal D_{\textrm{VB},i}}$ and $P_{\mathcal D_{\mathbb R^{2d_i}}}$ from the tangent pullback bundle $T(\breve F_i^*T(\mathbb R^{2d_i}\times M))$ onto these distributions. As with the subbundles, these distributions are smooth and globally well-defined due to their vertical bundle structures.

Now as before, we use the local PBDS construction above to form for each task a section of the tangent pullback bundle associating each robot position/velocity pair with the corresponding desired task pullback acceleration:
\begin{align*}
S_i : TM &\longrightarrow T(\breve F_i^*T(\mathbb R^{2d_i}\times M))\\
(p,v) &\mapsto P_{\mathcal D_{\mathbb R^{2d_i}}} \circ P_{\mathcal D_{\textrm{VB},i}}(\gamma'_{(p,v)}(0))\\
&= (((p,v),a), (0,(0,\gamma'^a_{(p,v)}(0),0))).
\end{align*}

Next we map robot accelerations and jerks to their corresponding task tangent pullback bundle accelerations:
\begin{align*}
Z_i : TTTM &\longrightarrow T(\breve F_i^*T(\mathbb R^{2d_i}\times M))\\
(((p,v),a),\kappa) &\mapsto P_{\mathcal D_{\mathbb R^{2d_i}}} \circ P_{\mathcal D_{\textrm{VB},i}}(d(\breve F_i^* d\breve F_i)_{((p,v),a)}(\kappa)).
\end{align*}

Now given a weighting pseudometric $\breve w_i$ over $T(\mathbb R^{2d_i}\times M)$ for each task, we can use a further higher-order task map $\mathfrak F_i : TTM \longrightarrow T(\mathbb R^{2d_i}\times M) : ((p,v),a) \mapsto (\breve F_i(p,v), (d\breve F_i)_{(p,v)}(a))$ to form the pullback pseudometric $\mathfrak F_i^*\breve w_i$ over $TTM$.

Then we can form our function returning the robot acceleration policy output for a given robot position and velocity:
\begin{equation}
\label{eq:tb_optimization}
\begin{aligned}
    \zeta : TM &\longrightarrow \mathcal D \subset TTTM\\
    (p,v) &\mapsto \underset{\kappa \in \mathcal D_{(p,v)}}{\arg \min} \ \sum^K_{i=1} \frac{1}{2} \lVert Z_i(\kappa) - S_i(p,v)   \rVert^2_{\mathfrak F_i^*\breve w_i},
\end{aligned}
\end{equation}
where $\mathcal D$ is the globally well-defined affine distribution of $TTTM$ such that the subspace $\mathcal D_{(p,v)} \subseteq T_{((p,v),(v,a))}TTM$ satisfies $\kappa^v = v$ and $\kappa^v = \kappa^b = a$ componentwise for each $(((p, v), (v,a)), (\kappa^p, \kappa^v, \kappa^b, \kappa^a)) \in \mathcal D_{(p,v)}$. Note that because $w_{T\mathbb R^{2d_i}}$ depends only on the base points of $T\mathbb R^{2d_i}$, given a point $(p,v) \in TM$, the coordinates of the pullback pseudometric $\mathfrak F_i^*\breve w_i$ do not depend on the particular $((p,v),a) \in T_{(p,v)}TM$ at which $\mathfrak F_i^*\breve w_i$ is evaluated.
Thus in practice we can simply evaluate it at $((p,v),0) \in TTM$.

To derive the local form of this acceleration policy, notice that the differential $d(\breve F_i^* d\breve F_i)$ can be written locally as the Jacobian
$$
\bm J(\breve{\bm F}_i^* \bm d \breve{\bm F}_i) = \begin{bmatrix}
\bm I_m & 0 & 0 & 0\\
0 & \bm I_m & 0 & 0\\
\ddot{\bm J \bar{\bm f}}_i & \dot{\bm J \bar{\bm f}}_i & \dot{\bm J \bar{\bm f}}_i & \bm J \bar{\bm f}_i\\
\dot{\bm J \bar{\bm f}}_i & 0 & \bm J \bar{\bm f}_i & 0\\
0 & 0 & 0 & \bm I_m\end{bmatrix},
$$
where we denote $\dot{J \bar f}_i$ analogous to Appx. \ref{appxsec:single_pbds}.
From this we can see that given a point having componentwise identifications $q = (((p,v),(v,a)), (v,a,a,\kappa)) \in \mathcal D_{(p,v)}$ we can locally compute
\begin{align}
    &\lVert Z_i(q) - S_i(p,v) \rVert^2_{\mathfrak F_i^*\breve w_i} \\
    &= \lVert \bm J \bar{\bm f}_i(p)\bm a + \dot{\bm J \bar{\bm f}}_i(p,v) \bm v - \dot{\bm\gamma}^\kappa_{v_p,i}(0)\rVert^2_{\breve {\bm w}_i^a(\mathfrak F_i(v_p,0)))},\nonumber
\end{align}
where $\dot{\bm \gamma}_{v_p} = (\dot{\bm \gamma}_{v_p}^v, \dot{\bm \gamma}_{v_p}^a, \dot{\bm \gamma}_{v_p}^b, \dot{\bm \gamma}_{v_p}^\kappa, \dot{\bm \gamma}_{v_p}^M)$, $\breve {\bm w}_i^a$ is the lower-right quadrant of $\bm w_{T\mathbb R^{2d_i}}$, and we define $\mathfrak F^0_i : TM \longrightarrow T(\mathbb R^{2d_i} \times M) : (\breve F_i(p,v), (0, v))$.
Then adapting \eqref{eq:local_pullback_accel} and using the computed Christoffel symbols of \eqref{eq:constr_christoffel}, we can compute the components of $\dot{\bm \gamma}_{v_p,i}$ along the indices of interest $j = 2m+d_i+1,\ldots,2m+2d_i$. First for forces, we have
\begin{align*}
    \Big(P_{\breve F_i^*G}&\Big(\breve F_i^*\breve{\mathcal F}_i(\gamma_{v_p,i}(0))^\sharp\Big)\Big)^j\\
    &= g^{jh}_{\mathbb R^{2d_i}}(\hat F_i(v_p)) \bigg(\breve{\mathcal F}_{D,i}^h(\gamma_{v_p,i}(0)) - \frac{\partial \breve\Phi_i}{\partial x^h}(\breve F_i(v_p))\bigg).
\end{align*}
Then for the Christoffel symbols term, denoting $\bm \gamma_{v_p} = (\bm \gamma_{v_p}^v, \bm \gamma_{v_p}^a, \bm \gamma_{v_p}^M)$ and $\dot \alpha_{v_p} = (\dot \alpha_{v_p}^v, \dot \alpha_{v_p}^a)$ we have
\begin{align*}
   & P_{\breve F_i^*G}\Big(\dot\alpha_{v_p,i}^\ell(0) P_{\breve F_i^*G}\Big(\gamma_{v_p,i}(0)\Big)^h \breve F_i^* \breve \Gamma^{j-2m}_{\ell h}(v_p)\Big)\\
    &= P_{\breve F_i^*G} \Big(\dot\alpha_{v_p,i}^\ell(0) P_{\breve F_i^*G}\Big( (\breve F_i^*J\breve F_i)_{hr} \dot\alpha_{v_p,i}^r(0)\Big)\\
    &\quad (J\breve F_i)_{s\ell}(v_p)\breve\Gamma^{j-2m}_{sh}(\breve F_i(v_p))\Big)\\
    &= (\dot\alpha_{v_p,i}^v)^\ell(0) (J\bar f_i)_{hr}(p) (\dot\alpha_{v_p,i}^v)^r(0)\\
    &\hspace{30pt}(J\hat F_i)_{s \ell}(v_p) \breve\Gamma^{j-2m}_{s(h+d)}(\hat F_i(v_p))\\
    &+ (\dot\alpha_{v_p,i}^a)^\ell(0) (J\bar f_i)_{hr}(p) (\dot\alpha_{v_p,i}^v)^r(0)\\
    &\hspace{30pt}(J\hat F_i)_{s(\ell+d)}(v_p) \breve\Gamma^{j-2m}_{s(h+d)}(\hat F_i(v_p))\\
    &= (\Xi^v_i)_{(j-2m-d_i)\ell}(v_p) (\dot\alpha_{v_p,i}^v)^\ell(0)\\
    &\qquad+ (\Xi^a_i)_{(j-2m-d_i)\ell}(v_p) (\dot\alpha_{v_p,i}^a)^\ell(0),
\end{align*}
where we notice the Jacobian $\breve F_i^*J\breve F_i$ takes the local form
$$
\breve{\bm F}_i^* \bm J \breve{\bm F}_i = \begin{bmatrix} \dot{\bm J \bar{\bm f}}_i & \bm J \bar{\bm f}_i & 0\\
\bm J \bar{\bm f}_i & 0 & 0\\
0 & 0 & \bm I_m\end{bmatrix}.
$$

From this can compute solutions to \eqref{eq:tb_optimization} locally as
\begin{align*}
\bm\zeta(p,v) = \Bigg(\sum_{i=1}^K&\bm{\mathcal C}_i^\top(p,v) \breve{\bm w}_i^a(\mathfrak F_i(v_p,0)) \bm{\mathcal C}_i(p,v)\Bigg)^\dagger\\
    &\left(\sum_{i=1}^K \bm{\mathcal C}_i^\top(p,v)\breve{\bm w}_i^a(\mathfrak F_i(v_p,0))\bm{\mathcal{A}}_i(p,v)\right)
\end{align*}
$$
\quad \bm{\mathcal C}_i(p,v) = \bm J \bar{\bm f}_i(p) + \bm \Xi_i^v(p,v) \vspace{-2pt}
$$
\begingroup
\allowdisplaybreaks
\begin{align*}
    \bm{\mathcal A_i}(p,v) &= (\bm g_{\mathbb R^{2d_i}}^{-1})^a(\hat F_i(v_p)) \Big(\bm{\mathcal F}_{D,\mathbb R^{2d_i}}((d\hat F_i)_{v_p}(v_p,0))\\
    &- \nabla \bm\Phi_{\mathbb R^{2d_i}}(\hat F_i(v_p))\Big) - (\dot{\bm J \bar{\bm f}}_i(p,v)  - \bm \Xi_i^a(p,v) )\bm v,
\end{align*}
\endgroup
where $(\bm g_{\mathbb R^{2d_i}}^{-1})^a$ is the lower-right quadrant of $\bm g_{\mathbb R^{2d_i}}^{-1}$ and the Euclidean gradient $\nabla \bm\Phi_{\mathbb R^{2d_i}}$ is taken over the last $d_i$ components. Additionally we have
\begin{equation*}
\begin{gathered}
    (\Xi^v_i)_{j\ell}(p,v) = (J\hat F_i)_{s\ell}(v_p) (\Gamma_{\mathbb R^{2d_i}})^{j+d_i}_{s(h+m)}(\hat F_i(v_p)) J\bar f_{hr}(p) v^r\\
    (\Xi^a_i)_{j\ell}(p,v) = (J\hat F_i)_{s(\ell+d_i)}(v_p) (\Gamma_{\mathbb R^{2d_i}})^{j+d_i}_{s(h+m)}(\hat F_i(v_p)) J\bar f_{hr}(p) v^r.
\end{gathered}
\end{equation*}

This leads to the multi-task PBDS policy for tangent bundle tasks:
\begin{definition}[Multi-Task PBDS on Tangent Bundles] Let $\{f_i : M \longrightarrow N_i\}_{i=1,\ldots,K}$ be smooth task maps, where $N_i$ is embedded in $\mathbb R^d$. Then given corresponding Riemannian metrics $\breve g_i$ on $\mathbb R^{2d_i} \times M$, smooth potential functions $\breve \Phi_i : \mathbb R^{2d_i} \times M \longrightarrow \mathbb R_+$, dissipative forces $\breve{\mathcal F}_{D,i} : T(\mathbb R^{2d_i} \times M)\longrightarrow T^*(\mathbb R^{2d_i} \times M)$, and weighting pseudometrics $\breve w_i$ on $T(\mathbb R^{2d_i} \times M)$, the set $\{(f_i, g_i, \breve \Phi_i, \breve{\mathcal F}_{D,i}, \breve w_i)\}_{i=1,\ldots,K}$ forms a multi-task PBDS with curves $\sigma : [0,\infty) \longrightarrow M$ satisfying
\begin{equation}
\label{eq:appx_multitask_pbds}
\begin{cases}
\ddot{\sigma}(t) = \zeta(\dot{\sigma}(t)) = \\
\qquad \underset{\kappa \in \mathcal D_{\dot\sigma (t)}}{\arg \min} \ \sum^K_{i=1} \frac{1}{2} \lVert Z_i(\kappa ) - S_i(\dot\sigma(t))   \rVert^2_{\mathfrak F_i^*\breve w_i} \\
\dot\sigma(0) = (p_0, v_0).
\end{cases}\hspace{-10pt}
\end{equation}
\end{definition}
Similar existence, uniqueness, smoothness, and stability guarantees can be given for this dynamical system as for that of Appx. \ref{appxsec:single_pbds}, but these are outside the scope of this appendix. The practical use of such PBDS policies is discussed in Sec. \ref{subsec:task_priorities_vel_dep_metrics}.}{}
\end{document}